\documentclass[journal]{IEEEtran}
\usepackage[style=ieee,maxnames=4,minnames=3,maxbibnames=3]{biblatex}

\usepackage{amsfonts}
\usepackage[cmex10]{amsmath}
\usepackage{amsthm}

\DeclareMathOperator*{\argmin}{arg\,min}

\theoremstyle{definition}

\theoremstyle{remark}

\theoremstyle{remark}
\newtheorem{assumption}{Assumption}

\theoremstyle{remark}
\newtheorem{lemma}{Lemma}
\newtheorem{theorem}{Theorem}
\theoremstyle{remark}

\usepackage{multirow}
\usepackage{array}
\usepackage{mdwmath}
\usepackage{mdwtab}
\usepackage{eqparbox}
\usepackage[caption=1]{caption}
\usepackage{stfloats}
\usepackage{url}
\usepackage{amssymb}
\usepackage{float}
\usepackage{indentfirst}
\usepackage{makeidx}
\usepackage{tabularx}
\usepackage{cases}
\usepackage{mathtools}
\usepackage{color}
\usepackage{bm}
\usepackage{stfloats}
\usepackage{lipsum}
\usepackage{amsfonts}
\usepackage{bm}
\usepackage{dsfont}
\usepackage{lipsum}
\usepackage{graphicx}
\usepackage{url}
\usepackage{hyperref}
\usepackage{bbm}
\newcommand{\RNum}[1]{\uppercase\expandafter{\romannumeral #1\relax}}
\usepackage{ulem}

\makeatletter

\newif\if@restonecol
\makeatother

\usepackage[linesnumbered,ruled,vlined]{algorithm2e}
\usepackage{algpseudocode}
\usepackage{amsmath}

\makeatother

\usepackage{amsthm}
\begin{document}
\captionsetup{font={footnotesize}}
\captionsetup[figure]{labelfont={rm},labelformat={default},labelsep=period,name={Figure}}
\title{Communication-Efficient Federated Learning with Binary Neural Networks}

\author{
        Yuzhi~Yang,~\IEEEmembership{Student Member,~IEEE},~
        Zhaoyang~Zhang,~\IEEEmembership{Senior Member,~IEEE},\\
        and Qianqian~Yang,~\IEEEmembership{Member,~IEEE}
\thanks{Y.~Yang, Z.~Zhang (Corresponding Author) and Q.~Yang are with the College of Information Science and Electronic Engineering, Zhejiang University, Hangzhou 310007, China, and with the International Joint Innovation Center, Zhejiang University, Haining 314400, China, and also with Zhejiang Provincial Key Lab of Information Processing, Communication and Networking (IPCAN), Hangzhou 310007, China. (e-mails: \{yuzhi\_yang, ning\_ming, qianqianyang20\}@zju.edu.cn)}
\thanks{This work was supported in part by National Key R\&D Program of China under Grant 2020YFB1807101 and 2018YFB1801104, and National Natural Science Foundation of China under Grant U20A20158, 61725104 and 61631003.}
}

\date{}
\maketitle
\begin{abstract}
Federated learning (FL) is a privacy-preserving machine learning setting that enables many devices to jointly train a shared global model without the need to reveal their data to a central server. However, FL involves a frequent exchange of the parameters between all the clients and the server that coordinates the training. This introduces extensive communication overhead, which can be a major bottleneck in FL with limited communication links. In this paper, we consider training the binary neural networks (BNN) in the FL setting instead of the typical real-valued neural networks to fulfill the stringent delay and efficiency requirement in wireless edge networks. We introduce a novel FL framework of training BNN, where the clients only upload the binary parameters to the server. We also propose a novel parameter updating scheme  based on the Maximum Likelihood (ML) estimation that preserves the performance of the BNN even without the availability of aggregated real-valued auxiliary parameters that are usually needed during the training of the BNN. Moreover, for the first time in the literature, we theoretically derive the conditions under which the training of BNN is converging. { Numerical results show that the proposed FL framework significantly reduces the communication cost compared to the conventional neural networks with typical real-valued parameters, and the performance loss incurred by the binarization can be further compensated by a hybrid method.}
\end{abstract}


\section{Introduction}
Federated learning (FL) enables machine learning models to be trained in a distributed way\cite{FL}\cite{FL2}. In an FL setting, a federation of edge devices, such as Internet of Things devices, mobile phones, etc., trains a shared global model collaboratively under the coordination of a central server while keeping all the training data locally at the edges. This alleviates the growing privacy concern of sharing the data and reduces the communication cost caused by sending extensive data to a central node.

A typical federated learning system with a star network topology, illustrated in Fig. \ref{fig::FL1}, consists of a central server and many clients, where each client holds a local dataset to train a shared global model in a coordinated manner iteratively. At each training iteration, each client trains the model with its local dataset and then uploads the updated parameters to the server. The central server combines all the uploaded parameters from all the servers by algorithms such as Federated Averaging (FA)\cite{FL}, which simply takes average of the received parameters from different clients, and then broadcasts the combined parameters to all the clients. A client will then update the local model with the received parameters, which completes one iteration of training.

\begin{figure} [!htp]
\vspace{-0.15 cm}
    \centering
       \includegraphics[width=0.8\linewidth]{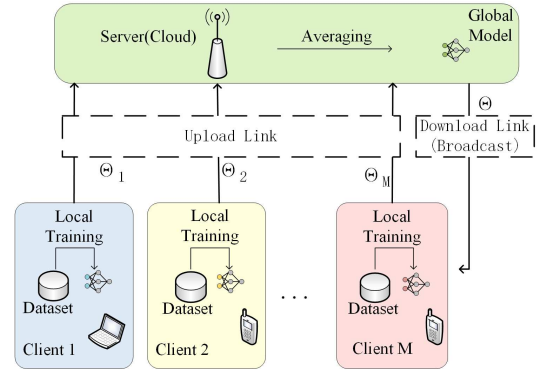}
	  \caption{A typical federated learning system with federated averaging algorithm\cite{FL}.}
	  \label{fig::FL1}
\vspace{-0.3 cm}
\end{figure}

Constantly updating the model parameters, which could be tens of millions for a deep neural network like ResNet\cite{resnet}, causes a massive communication overload. This could significantly slow down the convergence of FL as the communication links between the devices and the central server is usually with limited bandwidth and power\cite{FLP}. Therefore, many research efforts have been made to theoretically characterize the fundamental trade-offs between the performance of the trained models and the communication costs of the FL framework\cite{Zhang13}\cite{Braverman16}\cite{Han18}\cite{Acharya20}\cite{Barnes19}.
Meanwhile, from a more practical point of view, other efforts have been made to reduce the total communication cost.  Most of them focus on the uplink communication and attempt to either compress the parameters to be uploaded or reduce the frequency of uploading actions\cite{FLStrategies}\cite{Suresh17}\cite{Konecny18}\cite{Alistarh17}\cite{Horvath19}. Another direction of reducing the uplink communication overhead is to compress the gradient before sending to the server. Many algorithms have been proposed to this end, e.g., gradient sparsification\cite{Aji17}, gradient quantization\cite{Courbariaux15}\cite{TWN} and sparse ternary compression\cite{Sattler19}. Regarding the download communication, some existing methods also compressing the broadcast message from the server to the clients\cite{Akhaled19}.

These works have remarkably eased the communication burden of the FL systems, however, at the cost of degrading the model performance. This has motivated researchers to consider the models inherently with low precision parameters for the sake of less communication cost to upload these parameters and preserve the model performance at the same time \cite{FLStrategies}. The first attempt is made by J. Xu et al., who have proposed Federated Trained Ternary Quantization(FTTQ) in \cite{TC}, where ternary quantization is applied on all parameters of the model. Inspired by this work, we investigate the federated learning of the binary neural networks (BNN),  a type of neural network with only binary parameters, in this paper.

BNN was first proposed by I. Hubara et al. in 2016\cite{hubara_binarized_2016}, where they show the BNN achieves nearly the same performance as the corresponding neural networks with real-valued parameters, which has brought a lot of research attention since. During the training of a BNN, a set of auxiliary real-valued parameters is used to keep the accumulation of gradients and update the binary parameters accordingly. A trivial way of training BNN in an FL setting is for each client to upload the auxiliary real-valued parameters to the server to aggregate and then feedback the aggregated parameters to the clients. However, this results in the exact communication cost of training a real-valued neural network, which is against the initial motivation of employing BNN. To reduce the communication cost, it is preferable to send only the binary parameters to the server, which may degrade the performance of the BNN as the auxiliary real-valued parameters may not be correctly updated at each iteration of training. In this paper, we propose a novel federated learning algorithm of training the BNNs,  referred to as \textit{BiFL-BiML},  that only uploads the binary parameters to the server, and employs a novel maximum likelihood-based parameter updating approach, referred to as \textit{ML-PU}, to estimate the aggregated auxiliary parameters with the local auxiliary parameters and the received aggregation of binary parameters from all the clients. The main contributions of this paper are summarized as follows:

\begin{itemize}
  \item We propose a novel federated learning framework of training BNN, which uploads only binary values instead of 32-bits float points to the server such that the communication cost of the FL system is drastically reduced. A novel maximum likelihood-based approach is proposed to estimate the aggregated auxiliary parameters and update the local auxiliary parameters accordingly, which preserves the performance of the BNN training ever when only the binary parameters are aggregated during the training.
  \item To the best of our knowledge, we are the first to theoretically investigate the convergence of a general centralized BNN and derive a set of conditions under which BNNs are guaranteed to converge towards the optimum. We further extend the convergence analysis to the BNN under FL settings considered in this paper.
  \item Numerical results validate the effectiveness of the proposed method, which show it
  reduces the upload communication cost from 32 bits to 1 bit per parameters when comparing with the conventional FL implementation that sends the real-valued parameters in the uplink communication, while the performance loss incurred by the binarization can be further compensated by a hybrid method. Moreover, the proposed BiFL-BiML remarkably outperforms the other communication efficient FL implementations that upload the quantized parameters, including FTTQ\cite{TC} and other BiFL algorithms which do not utilize ML-PU to update the local auxiliary parameters, demonstrating the effectiveness of ML-PU.
\end{itemize}

This paper is organized as follows. In Section \ref{sec::pre}, we present the preliminaries about the federated learning and the BNN. We present the proposed federated learning framework of training BNN in Section \ref{sec::alg}, followed by the theoretical analysis on the convergence of the proposed method in Section \ref{sec::the}. Section \ref{sec::num} provides the numerical results and Section \ref{sec::con} concludes the paper.

\section{Preliminary}\label{sec::pre}

In this section, we first present the necessary preliminaries about federated learning for this work and then introduce the original version of BNN\cite{hubara_binarized_2016} and the BNN training method we employ in this work.

\subsection{Federated Learning}

We show an example of federated systems in Fig. \ref{fig::FL1}, which consists of several clients connected to a central server. The training data is distributed among all the clients either in an IID or a non-IID manner. A typical federated learning algorithm operates iteratively with four phases, i.e., the \textit{local training} phase , the \textit{parameter uploading} phase, the \textit{parameter aggregating} phase and \textit{parameter downloading} phase.

During the \textit{local training} phase, each client trains a common model shared by all the clients with its local data and by a particular neural network training algorithm. The \textit{parameter uploading} phase happens after the \textit{local training} phase when all the clients transmit the updated parameters to the central server. The communication cost occurring during this phase can be prohibitive as every client in the system needs to send a whole set of all parameters of the model to the server. After receiving all the local parameters from the clients, the server aggregates parameters in a certain way to generate one set of parameters of the model during the \textit{parameter aggregating} phase, which will then be broadcast to all the clients in the \textit{parameter downloading} phase. One widely used method to aggregate parameters is Federated Averaging (FA) algorithm\cite{FL}, which we use as a baseline to compare the novel parameter aggregating method proposed in this paper. We will explain the process of the FA algorithm in details in the sequel. After the \textit{parameter downloading} phase, each client then initiates the model with the received parameters from the server and starts a new iteration of local training.

Here, we show the details of the Federated Averaging algorithm. Without loss of generality, assume that the considered federated learning system is to solve a typical image classification problem with a labeled dataset $\mathcal{D}$. Note that we overlook the details of communication protocol as it is out of the scope of this paper and assume that all sent models can be received successfully.
Each client caches a non-overlapping subset of the dataset $\mathcal{D}$, denoted by  $\mathcal{D}_{i}$, where $i= 1, ..., M$, and $M$ is the number of clients in the system. $\mathcal{D}_{i}$ contains $|\mathcal{D}_{i}|$ image-label pairs denoted as $(\mathbf{x}_{ij},y_{ij})$, where $j = 1,...,|\mathcal{D}_{i}|$, $\mathbf{x}_{ij}$ represents the $j$th image at the $i$th client, and $y_{ij}$ denotes the corresponding label. We emphasize that the subsets of dataset at the clients are non-overlaping, i.e., $\mathcal{D}_{i} \cap \mathcal{D}_{j}=\emptyset$, $\forall  i \neq j$, and hence, $|\mathcal{D}|=\sum_{i=1}^{M}|\mathcal{D}_{i}|$. We denote the local parameters of the common model at client $i$ as $\bm{W}^{i}$, $i = 1,...,M$. The FA algorithm aggregates all the local parameters uploaded by the clients as follows,
\begin{equation}
\tilde{\bm{W}} = \frac{1}{|\mathcal{D}|}\sum_{i=1}^{M}|\mathcal{D}_{i}|\bm{W}^{i},
\end{equation}
where the derived set of parameters is the weighted sum of $\bm{W}^{i}$, $i = 1,...,M$ with the sizes of local datasets as the weights.

Without loss of generality, assume that all the clients employ the same loss function, given as $\ell(y, g(\bm{x},\bm{W}))$, where $y$ is the true label of the image $\bm{x}$, and  $g(\bm{x},\bm{W})$ denotes the predicted label by the model specified by the set of parameters $\bm{W}$. Hence we have the average loss of the local dataset at a certain client $i$, denoted by $F_i(\bm{W})$, as follows:
\begin{equation}
F_i(\bm{W}) = \frac{1}{|\mathcal{D}_{i}|}\sum_{j = 1}^{|\mathcal{D}_{i}|}\ell(y_{ij},g(\mathbf{x}_{ij},\bm{W})),
\end{equation}
and we then define the global loss as follows:
\begin{equation}
\begin{aligned}\label{BNN}
F_g(\bm{W}) &= \frac{1}{|\mathcal{D}|}\sum_{i=1}^{M}\sum_{j = 1}^{|\mathcal{D}_{i}|}\ell(y_{ij},g(\mathbf{x}_{ij},\bm{W}))\\
&=\sum_{i=1}^{M}\frac{|\mathcal{D}_{i}|}{|\mathcal{D}|}F_{i}(\bm{W}),
\end{aligned}
\end{equation}

The communication between the server and the clients takes place during the \textit{parameter uploading} phases and the \textit{parameter downloading} phases. The communication cost of the latter can be neglected as the transmission of the global parameters to all the clients can be done in a broadcast manner. Whereas in the \textit{parameter uploading} phase, each client has to unicast a different set of parameters to the server, which may introduce prohibitive communication costs. Hence in this paper, we focus on the minimization of the communication cost during this phase.

\subsection{Binary Neural Networks}\label{sec::bnn}

Binary neural network is a type of neural network with highly compressed parameters first introduced in \cite{hubara_binarized_2016}, where the neural network weights are binary, requiring much less storage space compared to the classical neural networks.

Consider a $L$-layer binary neural network, where the weights and activations are constrained to $+1$ or $-1$. We denote the weights (including trainable parameters of the activation) of each layer by $\bm{W}^b_l \in \{1, -1\}^*$, where $l =1, ..., L$ and $*$ represents the dimension of $\bm{W}^b_l$. We further introduce a scalar parameter $\vartheta_l$ for each layer acting as the trainable amplitude for this layer, which has been shown to improve the performance of BNN in \cite{XNOR}. The output of layer $l$ is given as follows:
\begin{equation}
\mathbf{a}_l = \vartheta_l f_l(\bm{W}^b_l, \mathbf{a}_{l-1})\label{layer_output}
\end{equation}
where $\mathbf{a}_{l-1}$ is output of layer $l-1$, and $f_l(\cdot)$ denotes the operation of layer $l$ on the input.

Here we explain how to train a BNN with the stochastic gradient descent (SGD) algorithm. Denote the learning rate as $\eta$. We introduce $\overline{\bm{W}}$ as the auxiliary real-valued parameters of the same size as $\bm{W}^{b}$. We employ the SGD algorithm to update the network parameters as well as $\bm{\vartheta}$ iteratively. The corresponding parameters at each iteration are specified by the subscript $t$, where $t= 1, ..., T$ are the total iterations.
At training iteration $t$,
\begin{subequations}\label{gradient}
\begin{align}
\overline{\bm{W}}_{t} \leftarrow \overline{\bm{W}}_{t-1} + \eta\frac{\partial \ell(\bm{W}^{b}_{t-1},\bm{\vartheta}_{t-1},\mathcal{D})}{\partial\bm{W}_{t-1}^{b}} \\
\bm{\vartheta}_{t} \leftarrow \bm{\vartheta}_{t-1} + \eta\frac{\partial \ell(\bm{W}^{b}_{t-1},\bm{\vartheta}_{t-1},\mathcal{D})}{\partial \bm{\vartheta}_{t-1}}
\end{align}
\end{subequations}
where $\mathcal{D}$ denotes the training data used in this iteration. $\ell(\cdot)$ is the loss function, the value of which depends on network parameters $\bm{W}^{b}$, $\bm{\vartheta}$ and the current training batch $\mathcal{D}^t$. Hence, the second terms on the right hand side of \eqref{gradient} correspond to the gradients on $\bm{W}^{b}$ and $\bm{\vartheta}$, respectively. Then, $\bm{W}^{b}$ is being updated at the end of iteration $t$ as follows:

\begin{equation}
\bm{W}_{t}^{b} \leftarrow \text{Sign}(\overline{\bm{W}}_{t}),
\end{equation}
where $\text{Sign}(\cdot)$ is an element-wise operation that returns $+1$ if the element is positive and $-1$ otherwise.

We further restrict the values of the auxiliary parameters $\overline{\bm{W}}_{t}$ within a given bound such that these parameters would not grow overly large whereas the corresponding binary parameters remain the same. Hence, at the end of the iteration, we have
\cite{hubara_binarized_2016}
\begin{equation}
\overline{\bm{W}}_{t} \leftarrow \text{clip}(\overline{\bm{W}}_{t}, -1, 1)
\end{equation}
where $\text{clip}(\cdot, 1, -1)$ is an element-wise operation that returns $1$ for any element larger than $1$, $-1$ for any element smaller than $-1$, or the same value for any element in between.

We emphasize that the auxiliary real-valued parameters $\overline{\bm{W}}$, are only used during the training to enable iterative optimization of the binary parameters of the BNN as shown above. During inference time, the output of each layer of the BNN is derived by the binary parameters $\bm{W}^{b}$ and the scalar parameter $\bm{\vartheta}$ as given by \eqref{layer_output}.

\section{Federated learning with binary neural networks}\label{sec::alg}

We consider a federated learning system where each client stores and trains a binary neural network of an identical structure to others' with local data, referred to as \textit{BiFL}. In this section, we first present different ways of implementing BiFL. Then for the communication efficient implementation, we propose a novel parameter updating approach that preserves the performance of the BNN as the communication-heavy counterparts.

\subsection{The Implementation of BiFL}\label{sec::FLBNN}
\begin{algorithm}[ht]
\caption{\textbf{BiFL algorithm}}
\label{alg::FL}

\textbf{Initialize:} All the clients load a neural network of the same architecture locally, and initialize the parameters of this neural network with the same parameters; each client informs the server the size of local dataset, i.e., $|\mathcal{D}_{i}|$, $i=1, ..., M$.\\
\For{$t = 1$ to $T$}
{
\{Start a global iteration\}\\
\{1. The \textit{local training} phase\}\\
\For{all clients $i = 1$ to $M$}{
\{All clients train a BNN simultaneously\}\\
Calculate the loss of BNN: $\ell(\bm{W}^{b,i}_{t-1},\bm{\vartheta}^i_{t-1},\mathcal{D}_i)$\\
Update local parameters by \eqref{gradient}\\
Binarization: $\bm{W}^{b,i}_{t}\leftarrow\overline{\bm{W}}^i_{t}$\\
}

\{2.The \textit{parameter uploading} phase\}\\
Each client upload $\overline{\bm{W}}^i_{t}$ or $\bm{W}^{b,i}_{t}$ to the server, $i = 1, ..., M$.\\
Each client also uploads $\bm{\vartheta}^i_{t}$ to the server, $i = 1, ..., M$.\\
\{3.The \textit{parameter aggregating} phase\}\\
The server obtains the global parameters by applying FA algorithm, i.e.,  $\bm{W}_t = \frac{1}{|\mathcal{D}|}\sum_{i=1}^{M}|\mathcal{D}_{i}|\overline{\bm{W}}^i_{t}$ or $\tilde{\bm{W}} = \frac{1}{|\mathcal{D}|}\sum_{i=1}^{M}|\mathcal{D}_{i}|\overline{\bm{W}}^{b,i}_{t}$\\
The server also obtains the global scalar parameters by applying FA algorithm, i.e., $\bm{\vartheta}_t = \frac{1}{|\mathcal{D}|}\sum_{i=1}^{M}|\mathcal{D}_{i}|\bm{\vartheta}^{b,i}_{t}$.\\
\{4.The \textit{parameter downloading} phase\}\\
Each client downloads $\bm{W}_t$ or $\tilde{\bm{W}}$ from the server, and use the local update method to update $\overline{\bm{W}}^i_{t}$.\\
Each client also updates the scalar parameters: $\bm{\vartheta}^i_t\leftarrow\bm{\vartheta}_t$.\\
Binarization: $\bm{W}^{b,i}_{t}\leftarrow\overline{\bm{W}}^i_{t}$\\

}

\end{algorithm}
We elaborate in this section how to train a binary neural network in a federated learning setting. The detailed implementation of BiFL is shown in Algorithm \ref{alg::FL}. Following the process of federated learning, each client first initializes the auxiliary real-valued parameters of the local binary neural network, denoted by $\overline{\bm{W}}$, and the scalar parameters for all the layers $\bm{\vartheta}^i_{0}$  where $i$ is the index of the client, $i=1, ..., M$. And then the initial $\bm{W}_{0}^{b}$ is derived by applying the Sign function on $\overline{\bm{W}_{0}}$, i.e.,
\begin{equation}
\bm{W}_{0}^{b, i} \leftarrow \text{Sign}(\overline{\bm{W}}_{0}).
\end{equation}

After the initialization, the system starts iterative training. During the \textit{local training} phase of iteration $t$, $\forall t \in [1:T]$, each client derives the updated parameters $\overline{\bm{W}}_{t}^i$, $\bm{\vartheta}^i_{t}$ and $\bm{W}_{t}^{b, i}$ following the training process described in Section \ref{sec::bnn} using its local dataset. After the local training, each client attempts to further update the local parameters collaboratively with others clients. One way is to obtain the globally aggregated data from the server by sending the local weights. The aggregated scalar parameters, denoted by $\bm{\vartheta}_{t}$, could be obtained by following the same process of traditional federated learning. That is, each client uploads $\bm{\vartheta}^i_{t}$, $i=1, ..., M$ during the \textit{parameter uploading} phase, and then the server aggregates these parameters utilizing the FA algorithm, i.e.,
\[\tilde{\bm{\vartheta}}_{t}=\frac{1}{|\mathcal{D}|}\sum_{i=1}^{M}|\mathcal{D}_{i}|\bm{\vartheta}^i_{t},\]
where we recall that $\mathcal{D}_{i}$, $i=1, ..., M$, are the disjoint local datasets at users, and $\mathcal{D}$ is the union of all the local datasets. Then the server broadcasts $\tilde{\bm{\vartheta}}_{t}$ and each client replaces $\bm{\vartheta}^i_{t}$ with $\tilde{\bm{\vartheta}}_{t}$, i.e., $\bm{\vartheta}^i_{t} \leftarrow \tilde{\bm{\vartheta}}_{t}$, $i= 1, ..., M$.

We present in the following different methods of further updating $\bm{W}_{t}^{b, i}$ and $\overline{\bm{W}}_{t}^i$ using global information after local training at each iteration of federated learning and illustrate the difference among them in Fig. \ref{fig::FLBNN}.

\begin{figure*} [!htp]
\vspace{-0.15 cm}
    \centering
       \includegraphics[width=0.9\linewidth]{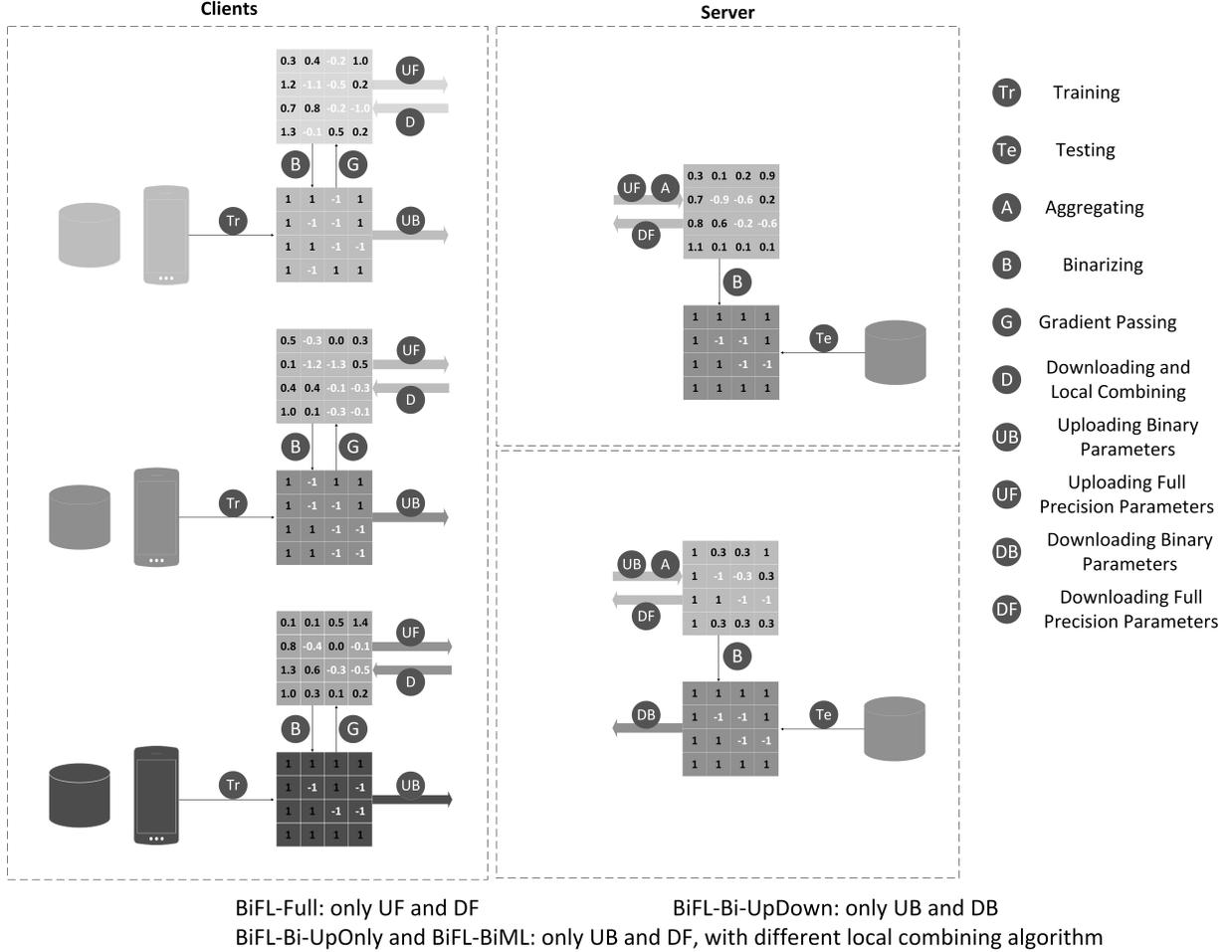}
	  \caption{Illustration of different implementations of BiFL. Only a layer is shown in the illustration representatively.}
	  \label{fig::FLBNN}
\vspace{-0.3 cm}
\end{figure*}

\begin{enumerate}
    \item The most straightforward way is for each client to upload $\overline{\bm{W}}_{t}^i$ to the server. Then following the same process of updating $\bm{\vartheta}^i_{t}$, the server derives the globally aggregated auxiliary real-valued parameters, denoted by $\bm{W}_{t}$, by using the FA algorithm, and then broadcast to all the clients. Each client then updates the local parameters as $\overline{\bm{W}}_{t}^i\leftarrow \bm{W}_{t}$, and $\bm{W}_{t}^{b, i}\leftarrow \text{Sign}(\overline{\bm{W}}_{t})$. We refer to BiFL with this method as BiFL-Full. We note that this way does not take advantage of the binarization of the network parameters, and end up with the same communication cost with typical neural networks.

    \item It is more desirable to upload only $\bm{W}_{t}^{b, i}$ than $\overline{\bm{W}}_{t}^i$ for the purpose of communication cost reduction. In this method, the server also aggregates the received $\bm{W}_{t}^{b, i}$ by the FA algorithm, that is, $\tilde{\bm{W}}_{t}=\frac{1}{|\mathcal{D}|}\sum_{i=1}^{M}|\mathcal{D}_{i}|\bm{W}_{t}^{b, i}$, where the aggregated parameters $\tilde{\bm{W}}_{t}$ is real-valued. Upon receiving $\tilde{\bm{W}}_{t}$ from the server, each client updates the local parameters as $\overline{\bm{W}}_{t}^i \leftarrow \beta\textrm{Sign}(\tilde{\bm{W}}_t)+(1-\beta)\overline{\bm{W}}_{t}^i$, $\bm{W}_{t}^{b, i}\leftarrow \text{Sign}(\overline{\bm{W}}_{t}^i)$, where $\beta$ is a predetermined parameter that ranges in $[0,1]$, which governs the significance of $\tilde{\bm{W}}_t$ on $\overline{\bm{W}}_{t}^i$.
    We refer to BiFL with this method as BiFL-Bi-UpOnly.

    \item This method, referred to as BiFL-Bi-UpDown, is similar to the method 2) where each client uploads the binary parameters. The only difference is that the server further binarizes the aggregated parameter, and broadcasts $\textrm{Sign}(\tilde{\bm{W}}_t)$ instead of $\tilde{\bm{W}}_t$. And the clients update local parameters by $\overline{\bm{W}}_{t}^i \leftarrow \beta\textrm{Sign}(\tilde{\bm{W}}_t)+(1-\beta)\overline{\bm{W}}_{t}^i$, $i=1, ..., M$. This method further reduces the communication cost compared to the former one. However, since the transmission of the aggregated parameters is in a broadcast manner, the reduction is not significant.
    Note that when $\beta=1$, BiFL-Bi-UpDown is equivalent to the algorithm proposed in \cite{TC} despite that it works on BNNs but not ternary neural networks.

    \item The last method we present here shares the same process as the Method 2) except that a novel maximum likelihood estimation based approach is employed to update $\overline{\bm{W}}_{t}^i$ using the received $\tilde{\bm{W}}_{t}$ and the local parameters $\overline{\bm{W}}_{t}^i$, which we will elaborate in details in the next section. We refer to the corresponding BiFL implementation as BiFL-BiML. The numerical results show that this approach achieves better convergence than BiFL-Bi-UpOnly and BiFL-Bi-UpDown.

\end{enumerate}

\subsection{Maximum Likelihood Estimation Based Parameter Updating}\label{MLE}
In this section, we present the proposed novel maximum likelihood estimation based parameter updating approach, referred to as ML-PU, in details. This approach, employed by BiFL-BiML, aims to use the received $\tilde{\bm{W}}_{t}$ from the server and the local parameters $\overline{\bm{W}}_{t}^i$ to approximate $\bm{W}_{t}$, i.e., the globally aggregated auxiliary real-valued parameter by using the FA algorithm. In such way, BiFL-BiML can train the local BNN models to have as good performance as by BiFL-Full while greatly reduces the communication cost. In the following, for the ease of presentation, we omit the subscript $t$ since the process presented below applies to any iteration. For simplicity, we assume that all clients hold the same amount of data, i.e. $|\mathcal{D}_{1}|=...=|\mathcal{D}_{M}|$. We will later show that the proposed approach can be easily generalized to scenarios where the amounts of data on each client are different.


We denote by $\overline{w}^i_j$ and $\tilde{w}_j$ the $j$th parameter within the local set of the auxiliary parameters at user $i$, i.e.,  $\overline{\bm{W}}^i$, and the corresponding $j$th parameter of the received aggregations  $\tilde{\bm{W}}$, respectively, $j =1, ..., N$, where $N$ is the total number of parameters in the BNN model. We consider the real-valued auxiliary parameters $\overline{\bm{W}}^i$ on each client as samplings of normally distributed random variables. More specifically, we assume that for any $j\in\{1,...,N\}$, the $j$th real-valued parameters at all $M$ clients, i.e., $\overline{w}^1_j,...,\overline{w}^M_j$, are independent samplings from the same normal distribution $\mathcal{N}(\mu_{j},\sigma_{j}^{2})$, where $\mu_{j}$ and $\sigma_{j}$ is the expectation and standard derivation of the normal distribution. We also assume that all the parameters $\overline{w}^i_1, ..., \overline{w}^i_j$ at client $i$, $i=1, ..., M$, are independently distributed.

The process of ML-PU is to estimate $\mu_{j}$ independently at each client with the local $\overline{w}^i_j$ and the received  $\tilde{w}_j$, and update the local parameter $\overline{w}^i_j$ with this estimation, denoted by $\hat{\mu}^i_{j}$, at the end of the iteration, i.e., $\overline{w}^i_j \leftarrow \hat{\mu}^i_{j}$. We note that since the local parameters at different clients are usually not identical, the estimations $\hat{\mu}^1_{j}$, ..., $\hat{\mu}^N_{j}$ are therefore of different values. As a result, the updated auxiliary parameters are also not identical across clients, different from the case with BiFL-Full. In the following, we elaborate how to obtain $\hat{\mu}^i_{j}$ with maximum likelihood estimation. For the ease of presentation, we omit the superscript $i$ and the subscript $j$ since the approach presented next applies to any $i \in \{1, ..., M\}$ and any $j \in \{1, ..., N\}$.

Upon receiving $\tilde{w}$, one can easily calculate the number of clients that have uploaded the binary parameters of value $+1$ instead of $-1$ to the server, denoted by $M_{P}$, which yields  $M_{P}=(\tilde{w}+1)M/2$. We denote the number of clients that have uploaded $-1$ by $M_{N}$, and have that $M_{N} = M-M_{P}$. With the knowledge of $M_P$ and $M_N$ and the local auxiliary parameter $\overline{w}$, each client uses the maximum likelihood criterion to estimate the average $\mu$ of the corresponding normal distribution as follows:
\begin{equation}\label{Prob}
\begin{aligned}
\arg\max_{\mu,\sigma}\quad & \binom{M-1}{M_P-\mathbbm{1}({\overline{w}>0})}\times\\
&\big(P\{x\geq0|x\sim\mathcal{N}(\mu,\sigma)\}\big)^{M_P-\mathbbm{1}({\overline{w}>0})}\times\\
&\big(P\{x<0|x\sim\mathcal{N}(\mu,\sigma)\}\big)^{M_N-\mathbbm{1}({\overline{w}<0})}
f(\tilde{w},\mu,\sigma),
\end{aligned}
\end{equation}
where $\mathbbm{1}$ denotes a indicator that equals to $1$ if the argument inside the brackets is true and $0$ otherwise. $f(\cdot,\mu,\sigma)$ represents the probability density function of normal distribution $\mathcal{N}(\mu,\sigma^{2})$, i.e. $f(\tilde{w},\mu,\sigma) = (2\pi\sigma^2)^{-1/2}e^{-(\tilde{w}-\mu)^2/\sigma^2}$.

We note that since $x\sim\mathcal{N}(\mu,\sigma)$, $(x-\mu)/\sigma\sim\mathcal{N}(0,1)$. We denote by $z(a)$ the upper $a$ quantile of the standard normal distribution, i.e. $z(a)=\int_{a}^{\infty}f(t,0,1)\mathrm{d}t$. We have
\begin{equation}\label{Prob2}
\begin{aligned}
	&P\{x\geq0|x\sim\mathcal{N}(\mu,\sigma)\}\\ =&P\{\frac{x-\mu}{\sigma}\geq-\frac{\mu}{\sigma}|\frac{x-\mu}{\sigma}\sim\mathcal{N}(0,1)\}\\
 =&z(-\frac{\mu}{\sigma})=1-z(\frac{\mu}{\sigma})
\end{aligned}
\end{equation}
We then apply the logarithmic transformation and the above transformation to (\ref{Prob}) and remove terms that are irrelative of $\mu$ and $\sigma$. The maximization problem then becomes
\begin{equation}
\begin{aligned}
	\arg\max_{\mu,\sigma}\quad &(M_P-\mathbbm{1}({\overline{w}>0}))\ln(1-z(\frac{\mu}{\sigma}))\\
	&+(M_N-\mathbbm{1}({\overline{w}<0}))\ln(z(\frac{\mu}{\sigma}))\\
&-\ln\sigma-\frac{(\overline{w}-\mu)^{2}}{2\sigma^{2}},\label{Prob3}
\end{aligned}
\end{equation}

We then add the term $1/2\ln\overline{w}^2$ to the objective function, and replace $\mu/\sigma$ by $u$ and $\overline{w}/\sigma$ by $v$, respectively, to simplify (\ref{Prob3}). It yields
\begin{equation}
\begin{aligned}
	\arg\max_{u,v}\quad
	&(M_P-\mathbbm{1}({\overline{w}>0}))\ln(1-z(u))\\
	&+(M_N-\mathbbm{1}({\overline{w}<0}))\ln(z(u))\\
&+\frac{1}{2}\ln v^2-\frac{(v-u)^{2}}{2},\\
s.t. \quad&\overline{w} v\geq0
\label{Prob4}
\end{aligned}
\end{equation}
which is equivalent to the optimization problem given in (\ref{Prob3}). The constraint is due to the fact that $\overline{w} v=\overline{w}^2/\sigma$ which should not be negative.

We take derivative of the objective function with regards to $v$, and obtain an analytical solution of it for the optimization problem, which is
\begin{equation}
v=\frac{u+\textrm{Sign}(\overline{w})\sqrt{u^{2}+4}}{2}.\label{v}
\end{equation}

Replace $v$ with the derived analytical solution. The optimization problem (\ref{Prob4}) becomes
\begin{equation}
\begin{aligned}
\hat{u}\triangleq&\arg\max_{u}f(u)\\
=&(M_P-\mathbbm{1}({\overline{w}>0}))\ln(1-z(u))\\
&+(M_N-\mathbbm{1}({\overline{w}<0}))\ln(z(u))\\
&+\ln\frac{\sqrt{u^{2}+4}+u\textrm{Sign}(\overline{w})}{2}
-\frac{(\frac{u+\textrm{Sign}(\overline{w})\sqrt{u^{2}+4}}{2}-u)^{2}}{2}\\
=&(M_P-\mathbbm{1}({\overline{w}>0}))\ln(1-z(u))\\
&+(M-M_P-\mathbbm{1}({\overline{w}<0}))\ln(z(u))-\ln2\\
&+\ln(\sqrt{u^{2}+4}+u\textrm{Sign}(\overline{w}))
-\frac{(\sqrt{u^{2}+4}-u\textrm{Sign}(\overline{w}))^{2}}{8}.\label{Probu}
\end{aligned}
\end{equation}
Given the values of $M$, $M_P$ and the sign of $\overline{w}$, we can solve the above optimization problem and thus obtain $\hat{u}$ numerically.

Note that for the considered FL system, the number of clients, i.e., $M$, is fixed throughout the whole training process, whereas $M_p$ is usually different for each iteration. To speed up the estimation of $\mu$ during training, we aim to obtain the approximate analytic relation between $\hat{u}$ and $M_P$, denoted by $h(\cdot)$, such that $\hat{u} \approx h(M_P, M, \overline{w})$. Noting that $f(u)$ contains term $\textrm{Sign}(\overline{w}))$, we consider the two cases, i.e., $\overline{w}>0$ and $\overline{w}<0$, separately. For a given $M$, we first derive $\hat{u}-M_P$ curves for both cases by numerically solving \eqref{Probu}, and then derive $h(\cdot)$ by fitting these two curves. For instance, we derive the $\hat{u}-M_P$ curves for $M=100$ which are presented in Fig \ref{fig::curve} (a) and (b), for $\overline{w}>0$ and $\overline{w}<0$, respectively. We can easily observe that the curves are approximately logarithmic thus we fit the curves with the following analytic function
\begin{equation}\begin{aligned}
h(M_P, M=100, \overline{w})=
\begin{cases}
   a_{1}+a_{2}\ln(M_P+a_{3})       & \text{if } \overline{w}>0 \\
   a_{4}+a_{5}\ln(-M_P+a_{6})       & \text{if } \overline{w}<0
  \end{cases}   \end{aligned}
\end{equation}
where $a_{1}$, $a_{2}$ and $a_{3}$ are parameters to be optimized to fit the $\overline{w}>0$ curve, while $a_{4}$, $a_{5}$ and $a_{6}$  for the $\overline{w}<0$ curve. From the symmetry of (\ref{Probu}), we know that $a_4=-a_1$, $a_5=-a_2$ and $a_6=M+a_3$, so we only needs to fit one curve. We also present the curves derived by the fitted analytic functions, for $\overline{w}>0$ and $\overline{w}<0$, in Fig \ref{fig::curve} (a) and (b), respectively, where we can observe that the difference between the curves derived numerically and the approximation curves is neglectable.

\begin{figure} [!htp]
\vspace{-0.15 cm}
    \centering
       \includegraphics[width=\linewidth]{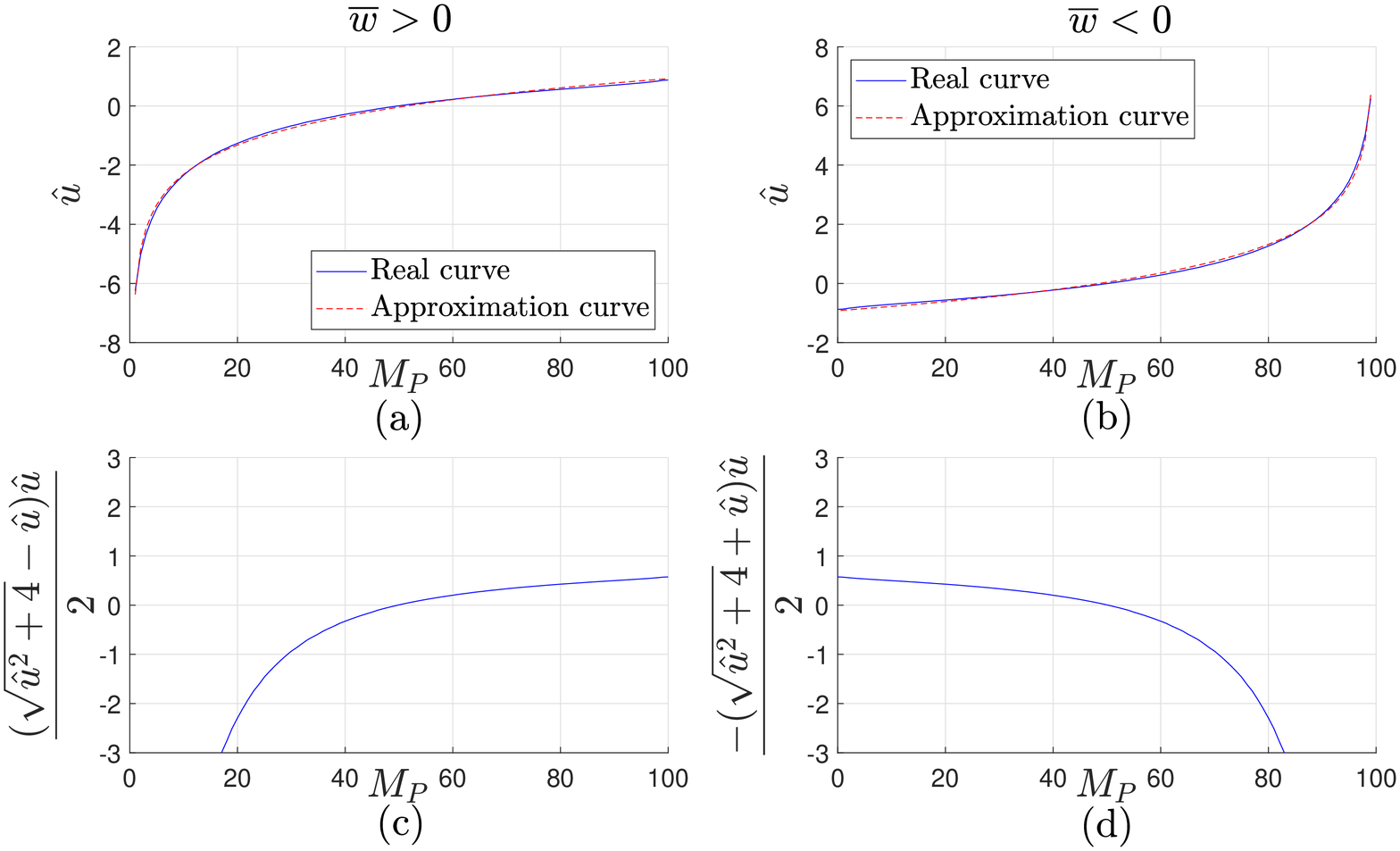}
	  \caption{(a) and (b): Real $\tilde{u}-M_P$ curve and least square approximation logarithmic curve when $M=100$, where $a_{1}=-5.4092$, $a_{2}=1.3761$ and $a_{3}=-0.5038$.
	  (c) and (d): $\frac{(\textrm{Sign}(\overline{w})\sqrt{\hat{u}^2+4}-\hat{u})\hat{u}}{2}-M_P$ curve when $M=100$}
	  \label{fig::curve}
\vspace{-0.3 cm}
\end{figure}

 According to \eqref{v}, we have the estimation of $v$ given by
\begin{eqnarray}
\hat{v} &=& \frac{\hat{u}+\textrm{Sign}(\overline{w})\sqrt{\hat{u}^2+4}}{2}.
\end{eqnarray}
Recalling that $u=\mu/\sigma$ and $v=\overline{w}/\sigma$, we have
\begin{eqnarray}
\hat{\mu}&=& \hat{u}\overline{w}/\hat{v}=\frac{(\textrm{Sign}(\overline{w})\sqrt{\hat{u}^2+4}-\hat{u})\hat{u}\overline{w}}{2}\label{Solution2}
\end{eqnarray}


As we can see in Fig. \ref{fig::curve} (c) and (d), $\frac{(\textrm{Sign}(\overline{w})\sqrt{\hat{u}^2+4}-\hat{u})\hat{u}\overline{w}}{2}$, i.e., $\hat{\mu}/\overline{w}$ takes value in range $(-1,1)$ for a large range of $M_P$, which means using the estimation presented above ends up with $|\hat{\mu}|<|\overline{w}|$. This also implies that $\hat{\mu}$ is not an unbiased estimation of $\mu$ as $\mu$ is the average value of all $\overline{w}$ at clients. Hence by updating $\overline{w}$ with $\hat{\mu}$ at the end of each iteration, $\overline{w}$ may be converging to zero, and lead to the failure of training. To tackle the issue raised above, we introduce a hyperparameter $\alpha$ to our proposed algorithm and update local real-valued parameters by
\begin{equation}
\overline{w} \leftarrow \textrm{clip}(\alpha\hat{\mu},-1,1),\label{Solution3}
\end{equation}
where $\alpha$ is a constant larger than 1.
We note that as auxiliary parameters do not take part in the forward computation, the scaling of them has no impact on the network output. We also note that the auxiliary real-valued parameters during the training of a BNN model are accumulations of gradients scaled by the chosen learning rate. The impact of this scaling on the convergence will be discussed in Section \ref{sec::the1} and how to determine the value of $\alpha$ in Section \ref{sec::alpha}.

Recall that we assumed each client has the same amount of data at local. Below, we extend the proposed ML-PU method to a more general scenario where the sizes of datasets are not identical across clients.  To apply the proposed method, for each client $i$, we assume a virtual FL system with $M_{i}=|\mathcal{D}|/|\mathcal{D}_{i}|$ clients, where each client has a dataset of the $|\mathcal{D}_{i}|$. With the received $\tilde{\bm{W}}_{t}$, client $i$ therefore can update the local auxiliary parameters by the ML-PU method with $M = M_{i}$, $M_{P}=(\tilde{w}+1)M_{i}/2$. We note that $M_{i}$ may not be an integer, which, however, is still
valid to derive the update of the auxiliary parameters following the steps of ML-PU. We validate the effectiveness of this approach with numerical results presented in Section \ref{sec::num}.

\section{Theoretical analysis}\label{sec::the}
In this section, we investigates the conditions under which the parameters of binary models will converge towards the optimal, and also discuss the impact of hyperparameter $\alpha$ on the convergence.
\subsection{The Convergence of Binary Models}\label{sec::the1}
Here we prove the convergence of gradient descent algorithm used to train binary models. We first consider a model with real-valued parameters $\overline{\bm{W}}$, where each element of $\overline{\bm{W}}$ is clipped between $1$ and $-1$. We denote by $F(\bm{W})$ as the average loss on the dataset $\mathcal{D}$ by this model given the set of parameters $\bm{W}$, i.e. $F(\bm{W}) = \frac{1}{|\mathcal{D}|}\sum_{i = 1}^{|\mathcal{D}|}\ell(y_{i},g(\mathbf{x}_{i},\bm{W}))$. Recall that $\ell(y_i, g(\bm{x}_i,\bm{W}))$ is the loss between the true label, denoted by $y_i$ and the prediction, given by $g(\bm{x}_i,\bm{W})$, of the input sample $\bm{x}_i$.

Following the literature on the convergence of gradient based training algorithms, we make the following assumptions, which have been proven to be satisfied for squared-SVM and linear
regression models\cite{Wang}.

\begin{assumption}
$F(\bm{W})$ satisfies
\begin{enumerate}
          \item $F(\bm{W})$ is $\xi$-strong convex and first-order derivable, i.e., $F(\bm{W}_1)-F(\bm{W}_2)\leq\nabla F(\bm{W}_1)^T(\bm{W}_1-\bm{W}_2)-\frac{\xi}{2}\|\bm{W}_1-\bm{W}_2\|^2$ for any $\bm{W}_1$ and $\bm{W}_2$.
          \item $F(\bm{W})$ is $\rho$-Lipschitz, i.e. $\|F(\bm{W}_1)-F(\bm{W}_2)\|\leq\rho\|\bm{W}_1-\bm{W}_2\|$ for any $\bm{W}_1$ and $\bm{W}_2$.
          \item $F(\bm{W})$ is $\beta$-smooth, i.e., $\|\nabla F(\bm{W}_1)-\nabla F(\bm{W}_2)\|\leq\beta\|\bm{W}_1-\bm{W}_2\|$ for any $\bm{W}_1$ and $\bm{W}_2$.
        \end{enumerate}
where $\|\cdot\|$ denotes for the l-2 norm of a vector.\end{assumption}
We have the following lemma, referred to as Descent lemma, followed from Lemma 1.2.3 of \cite{proof}, the proof of which can also be found in \cite{proof}.
\begin{lemma}[descent lemma]
If $F(\bm{W})$ satisfies Assumption 1, then
\begin{equation*}
    \begin{aligned}
    F(\bm{W}_1)-F(\bm{W}_2)\leq&\nabla F(\bm{W}_1)^T(\bm{W}_1-\bm{W}_2)\\
    &-\frac{1}{2\beta}\|\nabla F(\bm{W}_1)-\nabla F(\bm{W}_2)\|^2
    \end{aligned}
\end{equation*}
\end{lemma}

Denote by $\bm{W}^b$ the product of applying sign function element-wisely to $\overline{\bm{W}}$, i.e., a binarized version of $\overline{\bm{W}}$, such that $\bm{W}^b\in\{-1,1\}^N$. We have Lemma \ref{lemma::1}, which specifies the math relation between $\bm{W}^b$, $\overline{\bm{W}}$ and any $\bm{V}^b\neq\bm{W}^b, \bm{V}^b\in\{-1,1\}^N$.

\begin{lemma}\label{lemma::1}
For $\bm{W}^b$ and $\overline{\bm{W}}$ defined as above and any $\bm{V}^b\neq\bm{W}^b, \bm{V}^b\in\{-1,1\}^N$, if there is $K$ non identical elements among all the $N$ dimensions in $\bm{W}^b$ and $\bm{V}^b$, we have the following two conclusions:
\begin{enumerate}
    \item $\|\bm{W}^b-\bm{V}^b\|\leq2\|\overline{\bm{W}}-\bm{V}^b\|$;
    \item $|\langle\bm{W}^b-\bm{V}^b,\overline{\bm{W}}-\bm{V}^b\rangle|\leq\arccos\sqrt{\frac{K}{N}}$,
\end{enumerate}
 where $\langle\cdot,\cdot\rangle$ refers to the angle between two vectors.
\end{lemma}
\begin{proof}
The proof of the lemma can be found in Appendix \ref{proof_lemma::1}.
\end{proof}

We then prove the convergence of $\bm{W}^b$ towards the optimal values, which are denoted by $\bm{W}^*$, i.e., \[\bm{W}^* \triangleq \argmin_{\bm{W}^b \in\{-1,1\}^N} F(\bm{W}^b)\] and by $\epsilon=\|\nabla F(\bm{W}^*)\|$ the norm of the gradient at $\bm{W}^*$. We denote by $\bm{W}^b_t$ the parameters of the binary model at time $t$, $\overline{\bm{W}}_t$ the corresponding auxiliary real-valued parameters, and $K_t$ the number of nonidentical elements in $\bm{W}^b_t$ and $\bm{W}^*$, then
\begin{equation}\label{equ::K}
\|\bm{W}^b_t-\bm{W}^*\|=2\sqrt{K_t}.
\end{equation}
Recall that the auxiliary parameters in the binary model are updated by $\overline{\bm{W}}_{t+1}=\overline{\bm{W}}_t-\eta\nabla F(\overline{\bm{W}}^b_t)$ at time $t$. Lemma 3 below gives a sufficient condition under which the auxiliary parameters converge to the optimal set of binary parameters $\bm{W}^*$.

\begin{lemma}\label{lemma::2} For a binary model described above. If the loss function $F(\cdot)$ meets Assumption 1 and the learning rate $\eta$ satisfies
\[0<\eta<\phi(\lambda_t,K_t),\]
where $\lambda_t\triangleq\|\nabla F(\bm{W}^b_t)\|/\epsilon$,
\begin{equation*}
    \begin{aligned}
    \phi(\lambda_t,K_t)=&\frac{\lambda_t-1}{\sqrt{\beta^3\lambda_t^2 N(\lambda_t^2+1)}}(\lambda_t\sqrt{\xi K_t}-\sqrt{\xi(N-K_t)}\\&-\lambda_t\sqrt{(\beta-\xi) (N-K_t)}
    -\sqrt{K_t(\beta-\xi)}),
    \end{aligned}
\end{equation*}
and $\beta$ and $\xi$ are defined in Assumption 1, then we have
\begin{enumerate}
    \item $\frac{K_t\xi-\epsilon}{\epsilon}\leq \lambda_t\leq\frac{K_t\beta+\epsilon}{\epsilon}$;
    \item $\|\overline{\bm{W}}_{t+1}-\bm{W}^*\|^2< \|\overline{\bm{W}}_t-\bm{W}^*\|^2-4(\phi(\lambda_t,K_t)-\eta)\eta\beta\xi K_t.$
\end{enumerate}
\end{lemma}
\begin{proof}
The proof of the lemma can be found in Appendix \ref{proof_lemma::2}.
\end{proof}

We then discuss the conditions given in Lemma 3 and convergence rate achieved by the gradient descent algorithm for binary models. Note that $\phi(\lambda_t,K_t)$ can be rewritten as
\begin{equation}
\begin{aligned}
\phi(\lambda_t,K_t)=&\frac{\lambda_t-1}{\lambda_t}\cos(\arccos{\sqrt{\frac{\xi}{\beta}}}+\arccos\sqrt{\frac{K_t}{N}}\\&+\arccos\frac{\lambda_t}{\sqrt{\lambda_t^2+1}}).
\end{aligned}
\end{equation}
We can easily observe that $\phi(\lambda_t,K_t)$ increases with both $\lambda_t$ and $K_t$. Hence there exists a pair of $(\lambda_0, K_0)$ where $\lambda_0=K_0 \xi/\epsilon-1$, such that $\phi(\lambda_t,K_t)>\eta$, for any $\lambda_t>\lambda_0$ and $K_t>K_0$. This satisfies the condition in Lemma 3 on $\phi(\lambda_t,K_t)$. Hence, with Lemma 3, we can conclude that the auxiliary parameters can always converge to at least a suboptimal set of binary parameters if the loss function $F(\cdot)$ meets the conditions given in Assumption 1. Recall that $\epsilon\triangleq\|\nabla F(\bm{W}^*)\|$, which is close to zero since $\bm{W}^*$ is the optimal set of parameters. Also note that $\lambda_t\triangleq\|\nabla F(\bm{W}^b_t)\|/\epsilon$, hence $\lambda_t\rightarrow \infty$, which yields $\phi(\lambda_t,K_t)\approx\cos(\arccos{\sqrt{\frac{\xi}{\beta}}}+\arccos\sqrt{\frac{K_t}{N}})$. Note that when $K_t\approx\frac{(\beta-\xi)N}{\beta}$, we have $\phi(\lambda_t,K_t)\geq 0$. Hence, we can guarantee the model to converge to have at most  $\frac{(\beta-\xi)N}{\beta}$ different elements to the optimal solution. That is, $\frac{(\beta-\xi)N}{\beta}$ is an upper bound on difference between the trained model and the optimal solution. We have the following theorem showing the convergence rate of binary models by the gradient descent algorithm.

\begin{theorem}
\label{lemma::3}
The gradient descent algorithm for binary models provides  geometric convergence to a set of parameters $\bm{W}^b$  that are with at most $K_0$ elements different from the optimized { binary} parameters $\bm{W}^*$, that is, $\forall \kappa>0$, there exists a minimum $T_0$ such that for $t \geq T_0$
\[{ K_t=\|\bm{W}_{t}^b-\bm{W}^*\|/2}<K_0+\kappa,\]
and $T_0=\mathcal{O}(\log\frac{1}{K_0+\kappa})$ for any initial $\overline{\bm{W}}_0$.
\end{theorem}
\begin{proof}
The proof of the { theorem} can be found in Appendix \ref{proof_lemma::3}.
\end{proof}

We then discuss the impact of scaling the ML estimation with hyperparameter $\alpha$ to update the weights as shown in \eqref{Solution3} on the convergence of BiFL-BiML.
We note that this is equivalent to scaling the learning rate. Therefore, similarly, if $\alpha$ is too small, the model will be too slow to converge. On the other hand, when $\alpha$ is too big, the model may converge too quickly to a suboptimal solution. Hence we normally choose a value of $\alpha$ that is slightly larger than 1.

\subsection{The Convergence of Binary Models in FL Setting}\label{sec::alpha}
In this section, we discuss the convergence of BiFL-BiML, and explore the favorable range of the scaling parameter $\alpha$. We have shown that the ML estimation is not unbiased in Section \ref{MLE}. Recall that the introduction of $\alpha$ is to scale the estimation to be unbiased. The following theorem shows that BiFL-BiML converges if each parameter $\overline{w}$ is updated with an unbiased estimation in every iteration.

\begin{theorem}
If each parameter $\overline{w}$ is updated with an unbiased estimation and Assumption 1 holds,  BiFL-BiML is of geometric convergence in time and the number of clients $M$, that is, $\forall \kappa>0$, there exists a minimum $T_0$ such that for $t \geq T_0$
\[\|\bm{W}_{m,t}^b-\bm{W}^*\|/2<K_0+\kappa,\]
and $T_0=\mathcal{O}(\frac{1}{M}\log\frac{1}{K_0+\kappa})$ for any initial $\overline{\bm{W}}_{m,0}$.\end{theorem}
\begin{proof}
The proof of the theorem can be found in Appendix \ref{proof_lemma::4}.
\end{proof}

Note that it is intractable to obtain a close-form expression on the value of $\alpha$ that imposes unbiased estimation. As aforementioned in Section \ref{sec::the1}, $\alpha$ is preferably slightly larger than 1, and  numerical results reveals that the range of $\alpha$ leads to a satisfactory performance of the proposed algorithm is between $1.25$ and $2$.

\section{Numerical results}\label{sec::num}

To numerically evaluate the proposed federated learning framework with BNN, We consider a simple neural network shown in Fig.\ref{fig::BNN}, where we denote by ``Conv $m\times n\times n$'' a $m$-channel convolutional layer with a $n\times n$ kernel, ``FC $m$'' a fully connected layer with output size of $m$, "BN" a batch normalization layer, "$\downarrow$2" a pooling layer downsampling the input by a factor of 2, and both "Tanh" and "Softmax" the hyperbolic tangent and softmax activation function, respectively. We note that the last layer uses the softmax activation instead of tanh to get the predicted labels.
We evaluate the performance of the proposed method on three datasets. The first is the handwritten digit recognition dataset MNIST\cite{MNIST}, the second is the Federated Extended MNIST (FEMNIST) dataset\cite{FEMNIST}, which is a well-known dataset for federated learning, and the last is the CelebA\cite{FEMNIST}, which is a face recognition dataset based on the photographs of celebrities, and also a well-known dataset for federated learning. Since MNIST dataset is simpler, we use the smaller group of parameters in Fig.\ref{fig::BNN} for MNIST and the larger for FEMNIST. We use a ResNet with 4 convolutional layers and a full-connect layer for CelebA. The details of our neural network can be found in our code.\footnote{https://github.com/yuzhiyang123/FL-BNN}.
We compare the proposed model to other federated learning algorithms in terms of both the accuracy and the communication cost.

\begin{figure} [!htp]
\vspace{-0.15 cm}
    \centering
       \includegraphics[width=0.9\linewidth]{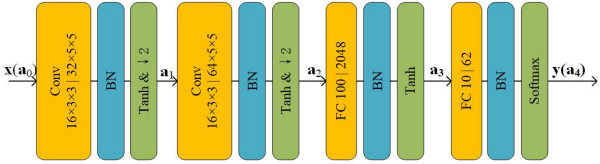}
	  \caption{The structure of the considered binary neural network}
	  \label{fig::BNN}
\vspace{-0.3 cm}
\end{figure}

\subsection{Settings}
We consider a star-topology federated learning system with one server and $M$ clients. We distribute our dataset to each of the clients. The algorithms we compare to are listed in the following, as well as the ways of distributing datasets to the clients.
\subsubsection{Compared algorithms}\label{sec::comparealg}
We compare the following algorithms numerically in this section, where we note that all of them employ the same neural network structure, the same learning rate, and are trained with the same optimizer. The neural network of BiFL is exactly the same as that of baseline and FA, except that it applies binarization on convolutional and full-connected layers.
\begin{itemize}
  \item Baseline: the centralized learning algorithm, where the whole dataset is located on a single computing center, and the neural network are with real-valued parameters.
  \item FA: the federated learning algorithm proposed in \cite{FL}, where the neural network are with real-valued parameters. The dataset is distributed among all the clients.
  \item BNN: the centralized learning algorithm as the baseline, but the neural network is with binary parameters instead, trained in the way as explained in Section \ref{sec::bnn}.
  \item FTTQ: the algorithm raised in \cite{TC}, which is a communication efficient FL algorithm for ternary neural networks.
  \item BiFL-Full, BiFL-Bi-UpOnly, BiFL-Bi-UpDown and  BiFL-BiML, which are different implements of federated learning system with BNN, which we explain in Section \ref{sec::FLBNN} in details.
\end{itemize}
\subsubsection{Dataset}
We use two datasets in the experiments. The first dateset we use is the MNIST dataset\cite{MNIST}, which consists of 70000 $28\times28$ gray-scale handwritten images. We divide it into a training set of 60000 images and a testing set of 10000 images, where all images of either set are evenly distributed in 10 classes. The second dataset we use is the Federated Extended MNIST (FEMNIST) dataset\cite{FEMNIST}, which is a well-known dataset for federated learning. The FEMNIST dataset is based on the Extended MNIST dataset, which is a more complicated { image} classification problem. { The FEMNIST dataset is based on the Extended MNIST dataset, which is a more complicated image classification problem. More specifically, the dataset contains images of handwritten numbers and English alphabets, which has 62 classes in total. It is recommended to use only a part of it to validate FL algorithms \cite{FEMNIST}, since the dataset is very large, and the distributed algorithms will not be well examined if the learning ability of the clients is too high. In our experiments, we randomly select 10\% of the Extended MNIST dataset for each experiment. For similar reason, we randomly select 50\% of the CelebA dataset for each experiment. Reference \cite{FEMNIST} also provides code for partitioning the FEMNIST and CelebA dataset by both IID and non-IID manners, and for evaluating methods. We evaluate our algorithm following their work.}
\subsubsection{Data distributions}\label{data}
We consider three way of distributing data to the clients: IID, non-IID and unbalanced data size, which we will explain in the following.
\begin{itemize}
  \item IID data: For MNIST, the dataset is evenly distributed to the clients with regards to the categories, i.e., all clients have the same amount of images from each class. For FEMNIST, we use the IID partition in \cite{FEMNIST}. We only reserve 500 { images} for each client to ensure that all clients have the same amount of data. We finally get 89 clients with 500 { images} on each client and a testing dataset of 5500 { images}.
  { To partition the CelebA dataset, we use the same way with FEMNIST, and get 100 clients with 500 images on each client and a testing dataset of 5000 images.}
  \item non-IID data: For MNIST, we divide the data from each class into $MN_c/10$ subdatasets. Then we distributed a random set of $N_c$ among the total $MN_c$ subdatasets to each client. We note that, in the case, each client ends up with no more than $N_c$ classes of images in the local dataset, and the total amount of images on each client is identical. For FEMNIST, we use the non-IID partition in \cite{FEMNIST}. Like in IID data, we only reserve 500 { images} for each client to ensure that all clients have the same amount of data. We finally get 89 clients with 500 { images} on each client and a testing dataset of 5500 { image}s.
  { To partition the CelebA dataset, we also use the same way with FEMNIST, and get 100 clients with 500 images on each client and a testing dataset of 5000 images.}
  \item Unbalanced data size: For MNIST, clients have different sizes of local datasets, while the amount of images from each category at each client is the same, where $20$ clients among $100$ evenly share 40\% of the dataset, another $40$ share another 40\% of the dataset, whereas the remaining $40$ share the remaining 20\% of the dataset.  For FEMNIST, we firstly divide the dataset into 178 subdatasets in an IID manner. We then distribute them on 89 clients in three categories similarly, and the clients in each category get 1, 2, and 4 subdatasets respectively. The amounts of clients in each category are 36, 35, and 18.
\end{itemize}
\subsubsection{Basic configuration}
We use the following configuration in our experiments.

\begin{itemize}
  \item Total amount of clients: $M=100$ for MNIST and $M=89$ for FEMNIST by default.
  \item Learning rate: {Decaying learning rate, \{0.005, 0.002, 0.001\} for MNIST, \{0.0003, 0.0001, 0.00005\} for FEMNIST{ , and \{0.1, 0.05, 0.02\}for CelebA}. The learning rate is decayed after the 30th and 60th epochs, respectively.}
  \item Gradient descent algorithm: Adam\cite{Adam}, a widely used gradient based optimizer.
  \item Batch size: $B=64$.
  \item Hyperparameter $\alpha$ in BiFL-BiML: 1.25 for MNIST and 1.75 for FEMNIST, since they leads to the best performance in terms of accuracy.
\end{itemize}
\subsection{Performance with IID Data}
In this section, we compare all the algorithms and evaluate the effect of the scaling parameter $\alpha$ on the IID data distribution.

\subsubsection{Performance comparison of different algorithms}
We compare the performance of all different BiFL implements in term of the accuracy and also the four benchmarks introduced above in Fig. \ref{fig::iid}. In Fig. \ref{fig::iid}, the shadows showing the standard deviation scaled by 1/5. We note that the training processes of FL and BiFL-Full are not shown completely in the figures of communication costs.
All the algorithms are evaluated on the same IID distributed datasets. We set $\beta=0.3$ for both BiFL-Bi-UpOnly and BiFL-Bi-UpDown, as this values leads to the best performance of both algorithms in terms of accuracy. We also include the numerical result of BiFL-Bi-UpDown with $\beta=1$ as it is a straightforward extension of the algorithm proposed in \cite{TC}. We also list the accuracy and the required communication costs by different algorithms for both uploading and downloading in Table \ref{table::iid}. All the data is averaged over 5 independently initialized dataset partitions. Note that the required downloading cost per parameter of our algorithm is only $\log(M+1)$ bit due to the fact that there are only $M+1$ values of the aggregated parameters as there are $M$ clients, which may increase if the dataset is not IID or there are more clients, but no more than 32 bit.

\begin{figure*} [!htp]
    \centering
       \includegraphics[width=0.95\linewidth]{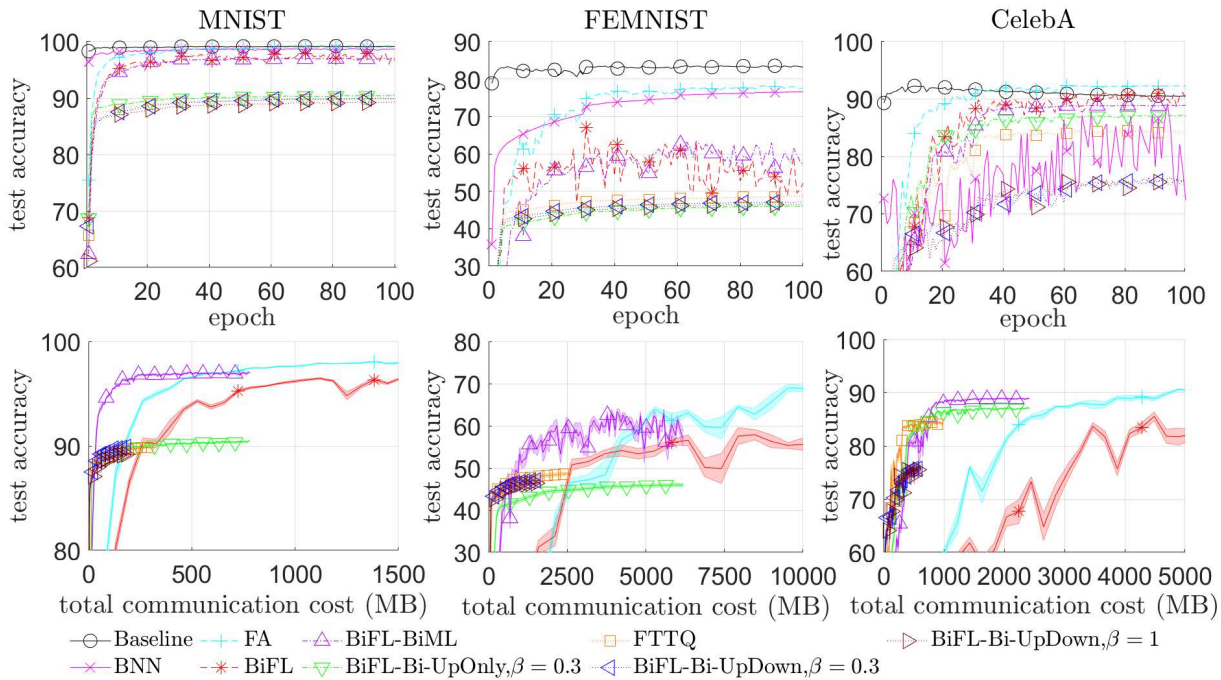}
	  \caption{ Accuracy by different algorithms on IID data distribution.}
	  \label{fig::iid}
\end{figure*}
\begin{center}

\begin{table*}[htbp]
\centering
\setlength{\tabcolsep}{1mm}{
\begin{tabular}{c|cc|cc|cc}
\multirow{2}{*}{algorithm} & \multicolumn{2}{c|}{MNIST} & \multicolumn{2}{c|}{FEMNIST}   & \multicolumn{2}{c}{CelebA}                                                                   \\ \cline{2-7}
                           & \begin{tabular}[c]{@{}c@{}}accuracy\\ (\%)\end{tabular} & \begin{tabular}[c]{@{}c@{}}upload/download\\ (MB per epoch)\end{tabular} &
                           \begin{tabular}[c]{@{}c@{}}accuracy\\ (\%)\end{tabular} & \begin{tabular}[c]{@{}c@{}}upload/download\\ (MB per epoch)\end{tabular} &
                           \begin{tabular}[c]{@{}c@{}}accuracy\\ (\%)\end{tabular} & \begin{tabular}[c]{@{}c@{}}upload/download\\ (MB per epoch)\end{tabular} \\ \hline
baseline                   & 99.25$\pm$0.04                                          & -                                                                        & 84.01$\pm$1.59                                          & -             &   92.27$\pm$0.23                                    & -                                                             \\
FA                         & 98.88$\pm$0.35                                          & 32.9/32.9                                                                & 76.65$\pm$0.15                                               & 264.3/264.3   &   89.19$\pm$0.84                                     & 99.2/99.2                                                             \\
BNN                        & 99.03$\pm$0.08                                          & -                                                                        &   78.53$\pm$2.99                                         & -           &   92.34$\pm$0.23                                     & -                                                               \\
BiFL-Full                & 98.24$\pm$0.07                                          & 32.9/32.9                                                                &  71.61$\pm$4.16                                       & 264.3/264.3     &   91.25$\pm$0.33                                    & 99.2/99.2                                                           \\
BiFL-BiML              & 97.17$\pm$0.21                                          & 1.0/6.8                                                                  &   66.19$\pm$2.83                                        & 8.3/53.6      &   88.97$\pm$0.39                                    & 3.1/21.1                                                             \\
BiFL-Bi-UpOnly, $\beta=0.3$     & 90.51$\pm$0.36                                          & 1.0/6.8                                                                  &  46.12$\pm$1.94                                        & 8.3/53.6    &   87.20$\pm$0.33                                    & 3.1/21.1                                                               \\
FTTQ                       & 89.99$\pm$0.28                                          & 1.6/1.6                                                                  &  48.81$\pm$1.99                                        & 13.1/13.1      &   84.50$\pm$1.05                                    & 5.0/5.0                                                            \\
BiFL-Bi-UpDown, $\beta=0.3$    & 90.09$\pm$1.48                                          & 1.0/1.0                                                                  &   47.27$\pm$2.44                                    & 8.3/8.3      &   76.38$\pm$1.01                                     & 3.1/3.1                                                             \\
BiFL-Bi-UpDown, $\beta=1$      & 89.41$\pm$0.55                                          & 1.0/1.0                                                                  &   46.65$\pm$1.98                                      & 8.3/8.3   &   76.02$\pm$2.86                                     & 3.1/3.1                                                                \\ \hline
\end{tabular}}
\caption{Accuracy and communication costs by different algorithms on IID data distribution.}
\label{table::iid}
\end{table*}
\end{center}

We can see that the centralized algorithms, i.e., the baseline and BNN, converges the fastest and results in the better accuracy than the federated learning algorithms as expected, however, there is a certain loss caused by binarization. We can also find a certain gap on the accuracy between all the distributed algorithms and corresponding centralized algorithms. It can be observed that the performance of our BiFL-BiML algorithm is { not too worse than} the performance of FA and BiFL-Full and exceeds the other three distributed algorithms, i.e., BiFL-Bi-UpOnly, FTTQ, and BiFL-Bi-UpDown, while reducing the uploading communication cost from 32 bits to 1 bit. We notice that BiFL-BiML also converges faster than the other three communication effecient FL algorithms. { We note that BiFL-Bi-UpDown and FTTQ require smaller download communication cost, however, the difference is neglectable due to the broadcasting, especially when the number of clients is big. The better performance of the BiFL-BiML algorithm in terms of accuracy demonstrates the significance of updating the local model with high precision aggregated parameters.}  Since there is still a certain gap in accuracy between BiFL-BiML and FL, we can use a simple hybrid method to overcome this, which will be discussed later.

\subsubsection{Effect of the scaling parameter}
We compare the performance in terms of accuracy when different values of scaling parameter $\alpha$ are used on MNIST, as shown in Fig. \ref{fig::alpha}. The result demonstrates that with a larger $\alpha$, our algorithm results in a lower accuracy, which validates our analysis. It also demonstrates that our algorithm may not converge to a good result when $\alpha$ is very small.
\begin{figure} [!htp]
\vspace{-0.15 cm}
    \centering
       \includegraphics[width=0.9\linewidth]{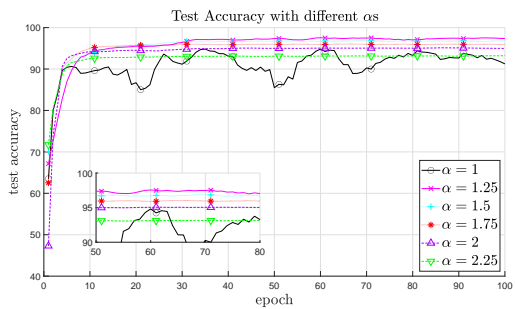}
	  \caption{Accuracy by BiFL-BiML with different values of $\alpha$.}
	  \label{fig::alpha}
\vspace{-0.3 cm}
\end{figure}
\subsubsection{Performance on hybrid algorithm}
\begin{figure} [!htp]
\vspace{-0.15 cm}
    \centering
       \includegraphics[width=\linewidth]{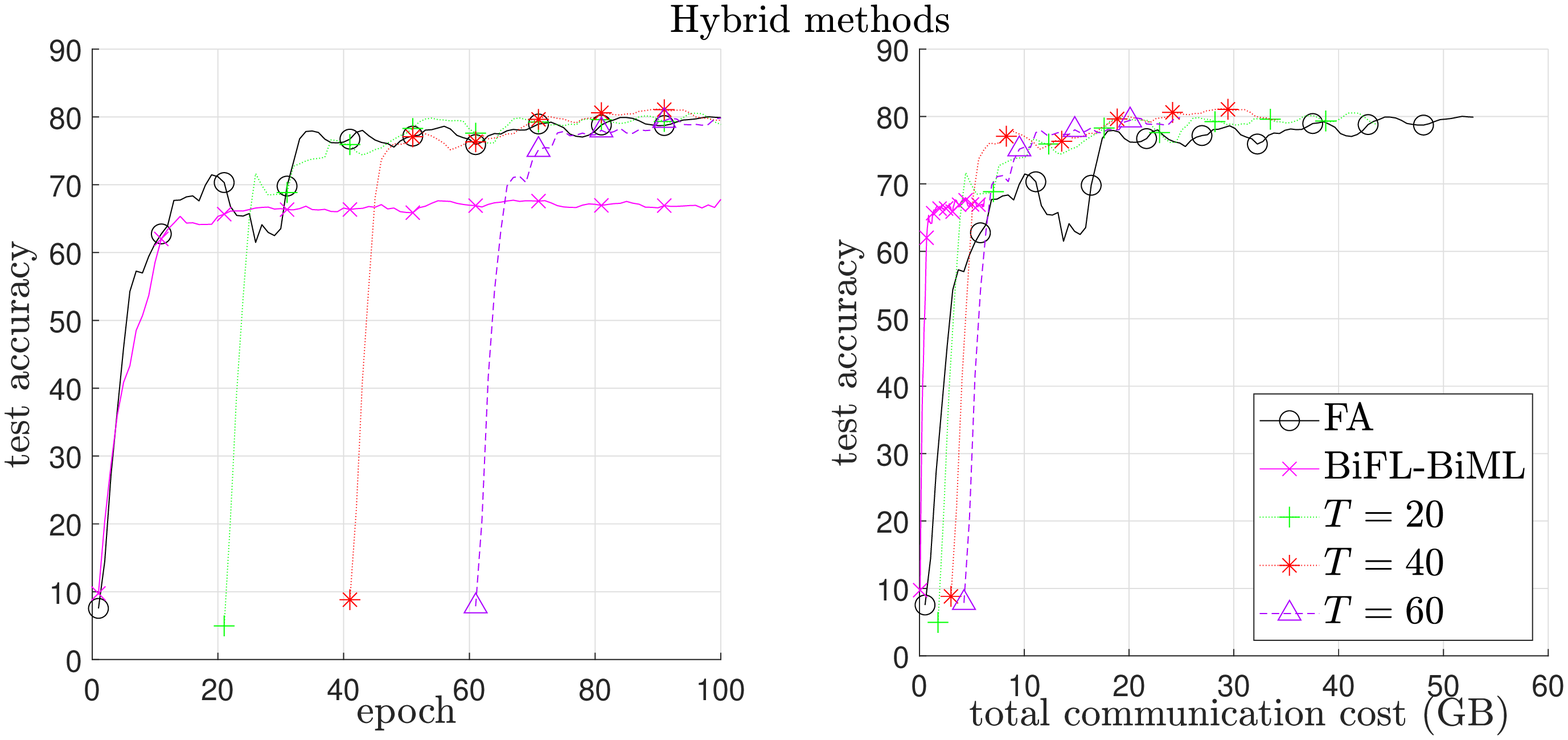}
	  \caption{Accuracy by the hybrid algorithm with different values of $T$ { on IID data distribution}.}
	  \label{fig::hybrid}
\vspace{-0.3 cm}
\end{figure}
{ Although our BiFL-BiML algorithm performs well in terms of accuracy comparing with other quantization schemes, there is still a notable performance gap comparing with that of FA, which is more significant with the FEMNIST dataset. To achieve a closer performance to FA and maintain a low communication cost at the same time, we train the NN in a hybrid way, in particular, by BiFL-BiML for the initial $T$ epochs, and by FA for the rest. We show the results on FEMNIST} when $T=\{20, 40, 60\}$ in Fig. \ref{fig::hybrid}. From Fig. \ref{fig::hybrid}, we can observe that the hybrid algorithm can achieve the performance of FA with lower communication cost. We also note that the test accuracy drops significantly at time $T$ when switching the algorithm, but converges to the training curve of FA in a short time. { This is due to the forward pass switching from using the binary parameters to using the real value parameter at time $T$. Due to the difference in the datasets on different clients, the models may also differs at the beginning, which is similar to the initialization of FA.} It can also be seen that a smaller $T$ results in faster convergence as expected with slightly larger communication cost for the same test accuracy.
\subsection{Effect of the Number of Clients}
We compare the result derived by the proposed algorithm and the FA algorithm when the number of clients in the system is different. In particular, we set the number of clients to be half, the same, and double of that of the IID settings. We present the smoothed results in Fig. \ref{fig::M}.
We can observe from Fig. \ref{fig::M} that the performance difference caused by the different amounts of clients is insignificant for all algorithms.
\begin{figure} [!htp]
\vspace{-0.15 cm}
    \centering
       \includegraphics[width=\linewidth]{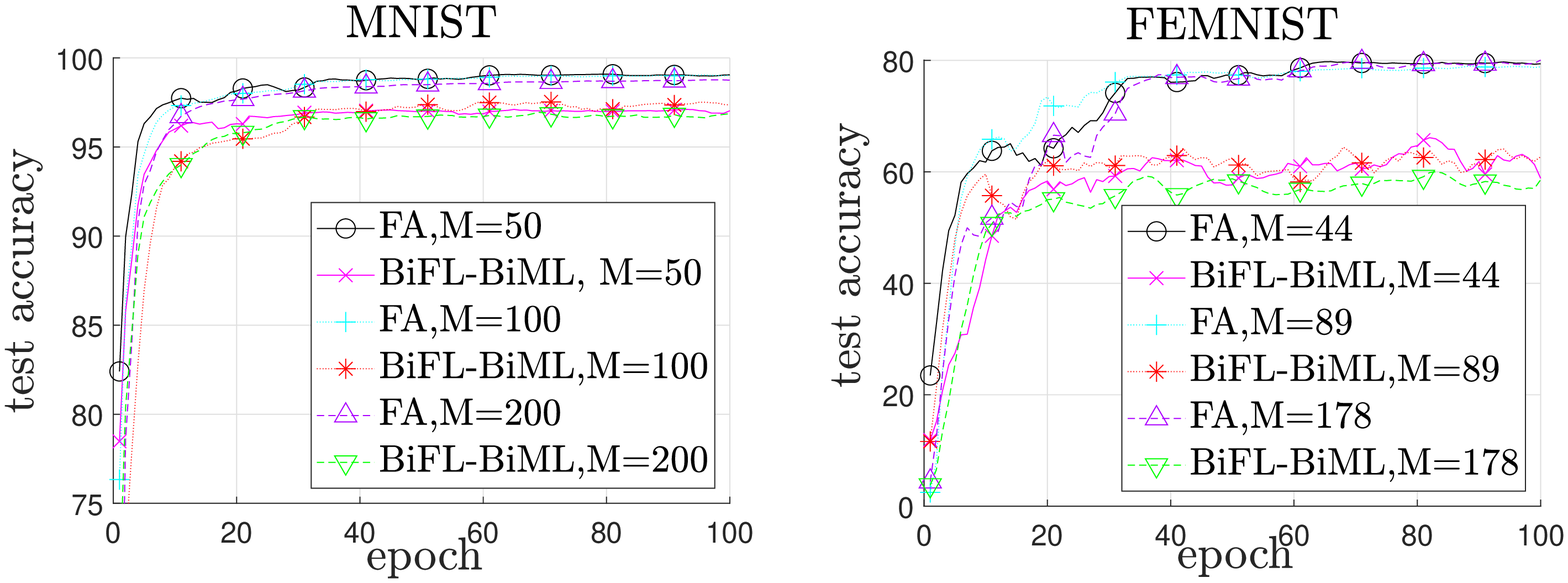}
	  \caption{Testing accuracy by BiFL-BiML and FA with different amounts of clients.}
	  \label{fig::M}
\vspace{-0.3 cm}
\end{figure}

\subsection{Performance with Non-IID Data}
\begin{table*}[]
\centering
\begin{tabular}{c|c|c|c|c}
algorithm     & MNIST,$N_c=5$  & MNIST,$N_c=3$  & FEMNIST & CelebA\\ \hline
FA            & 95.50$\pm$0.61 & 93.08$\pm$4.61 & 72.05$\pm$1.22    & 82.91$\pm$0.15          \\
BiFL-BiML & 86.00$\pm$1.87 & 90.05$\pm$3.87 & 60.38$\pm$2.38        & 81.09$\pm$0.20      \\
FTTQ          & 72.32$\pm$2.66 & 73.10$\pm$1.19 & 45.11$\pm$2.04    & 76.36$\pm$0.92          \\
BiFL-Bi-UpDown    & 73.39$\pm$2.13 & 73.26$\pm$4.28 & 44.19$\pm$3.22    & 68.51$\pm$1.14 \\\hline
\end{tabular}
\caption{ Accuracy by different algorithms on non-IID data distributions (in terms of \%).}
\label{table::niid}
\end{table*}
We compare different algorithm when the data is distributed to the clients in a non-IID manner as described above shown in Table \ref{table::niid}. All the results are derived by 5 independent experiments with independently initialized dataset partitions. We recall that the $N_c$ is the largest number of data classes each client can have according to our data partition method.
 { From Table \ref{table::niid}, it can be observed that the proposed BiFL-BiML algorithm outperform the FTTQ and the BiFL-Bi-UpOnly algorithms in terms of accuracy even on non-IID data, whereas a remarkable performance gap to that of FA still presents.}
\subsection{Effect of Unbalanced Data Size}
\begin{figure} [!htp]
\vspace{-0.15 cm}
    \centering
       \includegraphics[width=\linewidth]{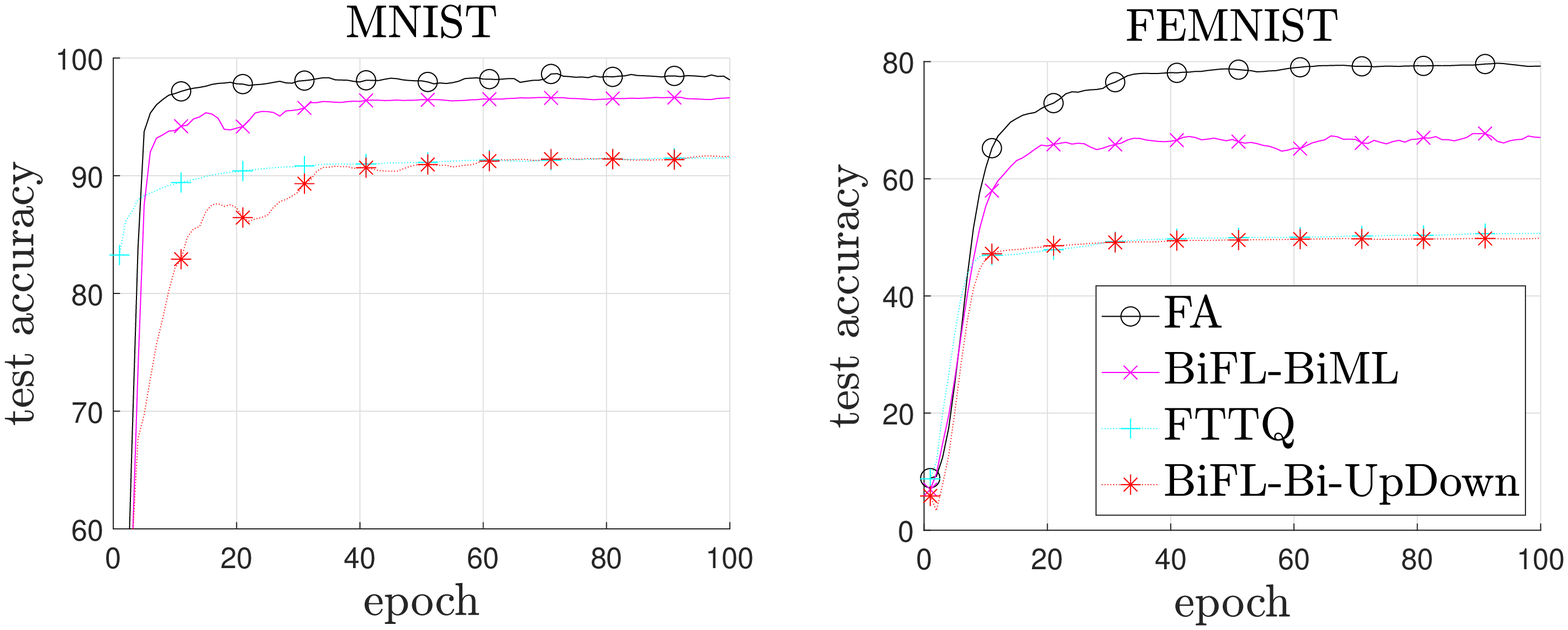}
	  \caption{Accuracy by different algorithms on unbalanced data size}
	  \label{fig::size}
\vspace{-0.3 cm}
\end{figure}
We evaluate the performance of the distributed algorithms when the datasets at clients are of different sizes. In particular, the data is allocated to each client in the way that is described above. As shown in Fig. \ref{fig::size}, we can find that the performance gap between the our BiFL-BiML algorithm and the FA algorithm is still small but more notable than the case with identical dataset size. Also note that BiFL-BiML algorithm still significantly outperform the FTTQ and BiFL-Bi-UpOnly algorithms.

\subsection{Effect of Partial Participation}
\begin{figure} [!htp]
\vspace{-0.15 cm}
    \centering
       \includegraphics[width=0.48\linewidth]{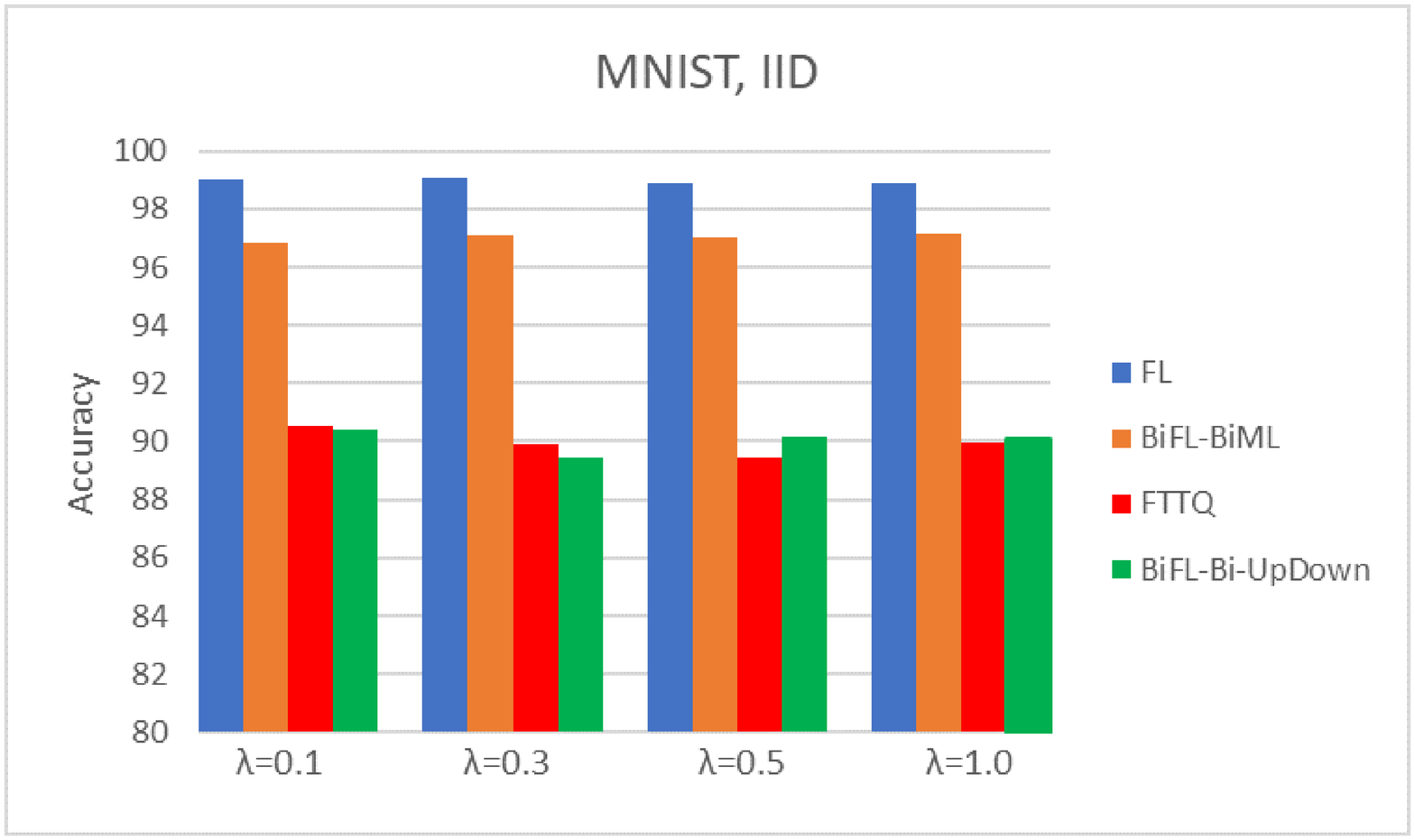}
       \includegraphics[width=0.48\linewidth]{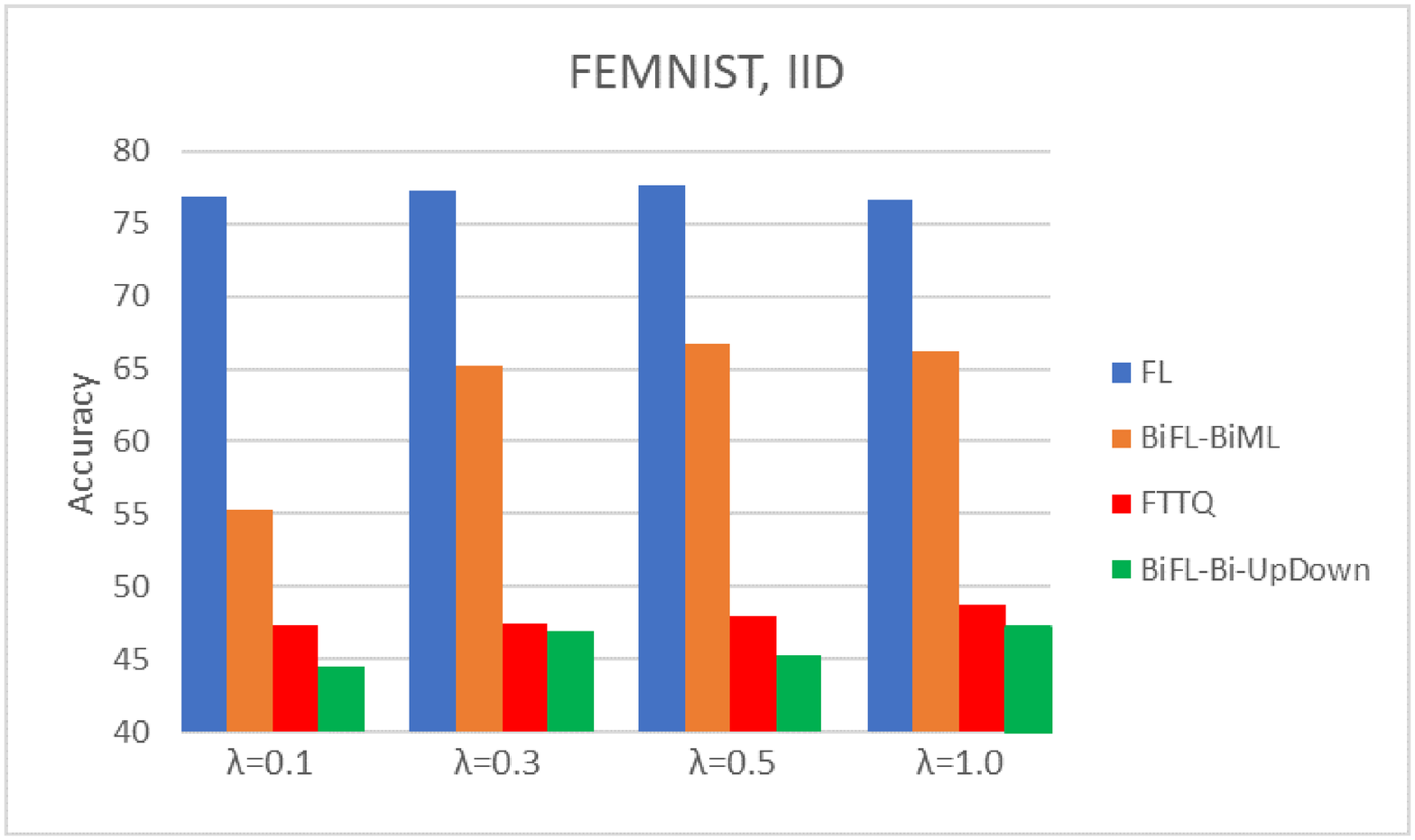}

       \includegraphics[width=0.48\linewidth]{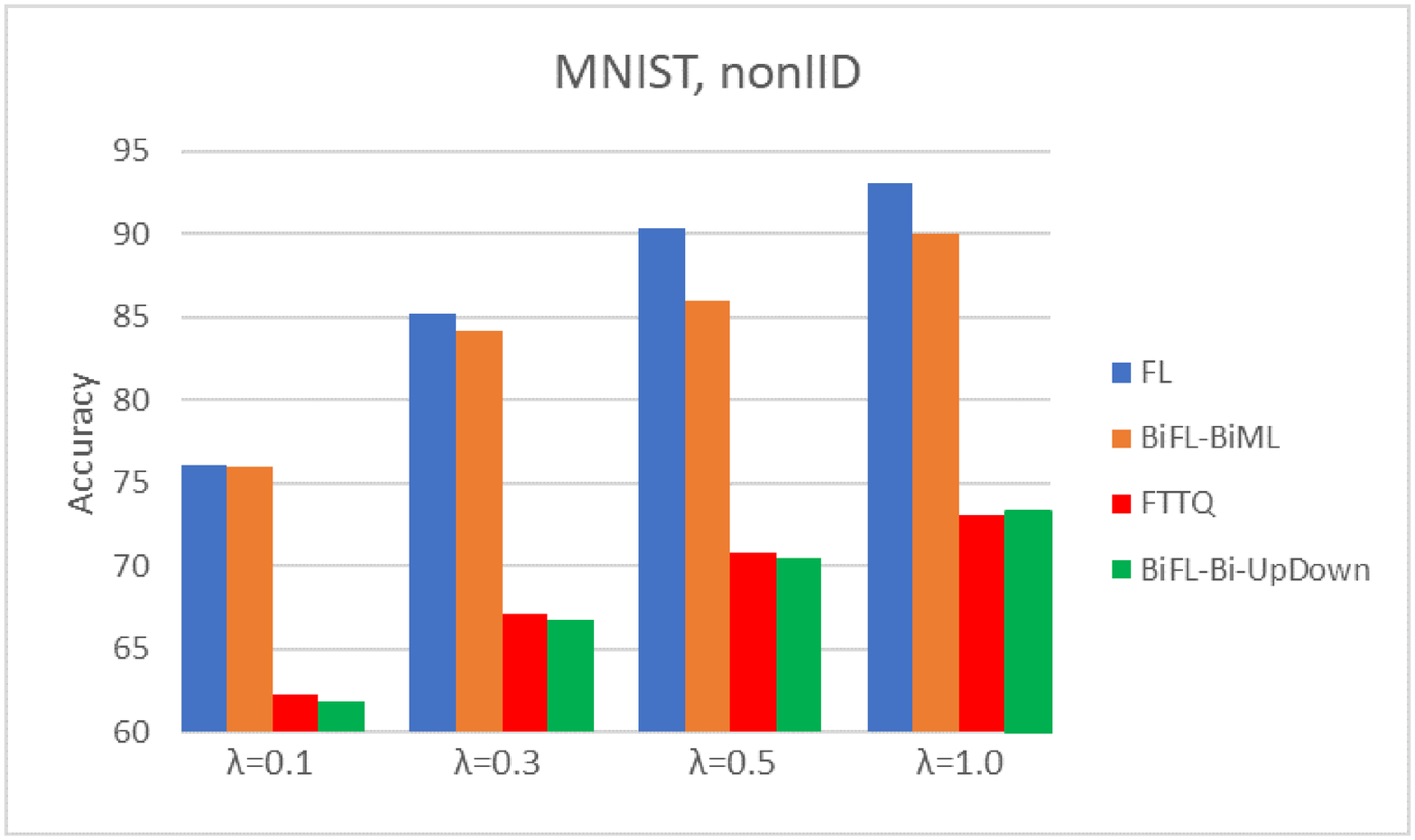}
       \includegraphics[width=0.48\linewidth]{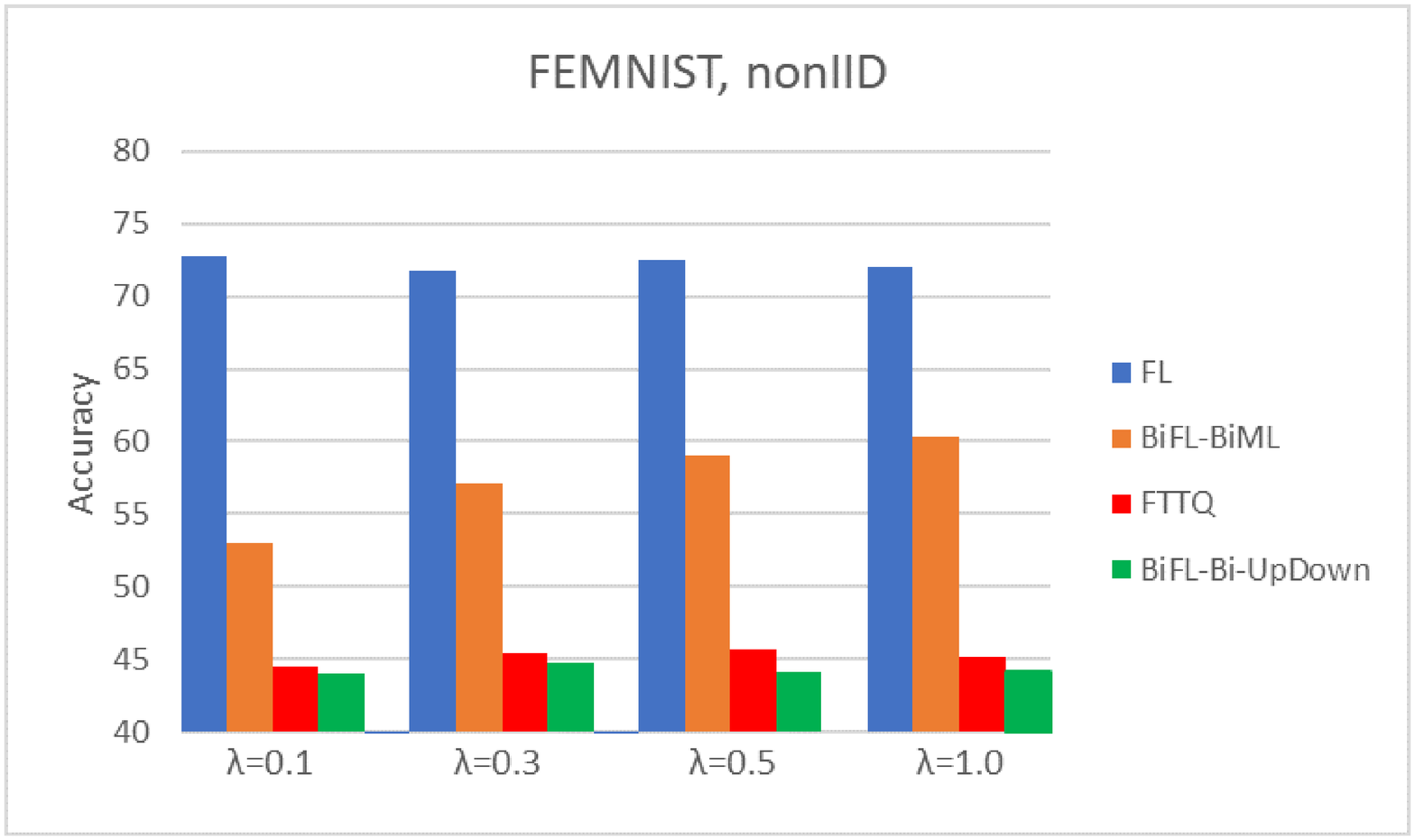}
	  \caption{Performance comparison with different participation rates.}
	  \label{fig::lambda}
\vspace{-0.3 cm}
\end{figure}
Here, we evaluate the performance of the distributed algorithms when not all the clients upload at each iteration. We denote the participation rate of the system as $\lambda$, i.e., only $\lambda M$ clients upload their parameters at each time slot. {  We show the results for both IID and non-IID data distributions in Fig. \ref{fig::lambda}. For MNIST, non-IID case, we set $N_c$=3. As we can see in Fig. \ref{fig::lambda}, from the results on MNIST, the participation rate of clients has marginal impact on the performance in terms of accuracy when using IID data distributions, whereas causing a notable drop in performance when $\lambda$ is small and using non-IID data distributions. From the results on FEMNIST, we can find that only the BiFL-BiML has obvious performance drop. This means that the data distribution of FEMNIST is not too far away from the IID data distribution, and that the BiFL-BiML is more sensible of partial participation mainly because the ML-PU is very sensible of the amount of the uploaded sets of parameters.}
\section{Conclusions and future directions}\label{sec::con}
In this paper, we proposed a novel federated learning algorithm, BiFL-BiML, to train neural networks with binary parameters. This algorithm introduces significant reduction in communication cost as the client uploads binary parameters instead of 32-bit float point parameters as in conventional federated learning systems. It employs a maximum likelihood estimation based parameter updating approach, ML-PU, in which each client estimates the auxiliary real-valued parameters based on the binary models' aggregation result received from the server and then updates its local parameters. To the best of our knowledge, we are the first to have theoretically derived the conditions on the convergence of training binary models by gradient descent algorithm. We numerically compared the proposed BiFL-BiML with the conventional FL algorithm, a communication-efficient FL algorithm for ternary neural works called FTTQ\cite{TC}, and two trivial implementations of FL learning with BNN, i.e., BiFL-Bi-UpOnly and BiFL-Bi-UpDown, respectively, which utilize simpler parameter updating mechanisms instead of ML-PU. The numerical results revealed that BiFL-BiML achieves performance close to the conventional FL algorithm while providing drastic reduction of communication cost, and significantly outperforms FTTQ, BiFL-Bi-UpOnly, and BiFL-Bi-UpDown under all the IID data, non-IID data, and unequal dataset size scenarios. The numerical results also revealed that the neural network can achieve nearly the same accuracy with a reduction of communication cost by a hybrid implementation of BiFL-BiML and the conventional FL algorithms.

\appendix
\begin{appendices}
\subsection{Proof of Lemma \ref{lemma::1}}\label{proof_lemma::1}
We first prove 1). Denote $\overline{w}_i$, $w_i^b$ and $v_i^b$ as the $i$th element of $\overline{\bm{W}}$, $\bm{W}^b$ and $\bm{V}^b$, respectively. Define $\mathcal{I}\triangleq\{i| w^b_i\neq v^b_i, 0<i\leq N\}$. $\mathcal{I}$ has $K$ elements in total. Recall that the real-valued parameters $\overline{\bm{W}}$ are clipped between -1 and 1, i.e. $-1\leq \overline{w}_i\leq 1$. We can find that for any $i\in\mathcal{I}$, $w^b_i$ and $\overline{w}_i$ have the same sign but the opposite to $v^b_i$. Therefore, we have $1/2|w^b_i-v^b_i|=1\leq|\overline{w}_i-v^b_i|<2$. We also have for any $i\notin\mathcal{I}$, $1/2|w^b_i-v^b_i|=0\leq|\overline{w}_i-v^b_i|$. Combine both cases and it yields  $1/2\|\bm{W}^b-\bm{V}^b\|\leq\|\overline{\bm{W}}-\bm{V}^b\|$, which completes the proof of conclusion 1) in Lemma \ref{lemma::1}.

We prove 2) next. we have
\begin{equation}
\begin{aligned}
(\bm{W}^b-\bm{V}^b)^T(\overline{\bm{W}}-\bm{V}^b)&=\sum_{i=1}^{N}(w^b_i-v^b_i)(\overline{w}_i-v^b_i)\\
&=\sum_{i\in\mathcal{I}}(w^b_i-v^b_i)(\overline{w}_i-v^b_i)\\
\end{aligned}
\end{equation}
Notice that $(w^b_i-v^b_i)(\overline{w}_i-v^b_i)>0$. Hence, we can apply Cauchy-Schwarz inequality, which yields
\begin{equation}\label{lemma2_1}
\begin{aligned}
(\bm{W}^b-\bm{V}^b)^T(&\overline{\bm{W}}-\bm{V}^b)\geq \sqrt{\sum_{i\in\mathcal{I}}(w^b_i-v^b_i)^2\sum_{i\in\mathcal{I}}(\overline{w}_i-v^b_i)^2}
\end{aligned}
\end{equation}
Also notice that
\begin{equation}\begin{aligned}
|\overline{w}_i-v^b_i|&=|w^b_i-v^b_i|-|\overline{w}_i-w^b_i|\\
&=2-|\overline{w}_i-w^b_i|\geq1, \forall i\in\mathcal{I}
\end{aligned}
\end{equation}
and
\begin{equation}
|\overline{w}_i-v^b_i|=|\overline{w}_i-w^b_i|\leq1, \forall i\notin\mathcal{I}.
\end{equation}
We have
\begin{equation}\label{lemma2_2}
\begin{aligned}
&\frac{N}{K}\sum_{i\in\mathcal{I}}(\overline{w}_i-v^b_i)^2\\
&=\sum_{i\in\mathcal{I}}(\overline{w}_i-v^b_i)^2+\frac{N-K}{K}\sum_{i\in\mathcal{I}}(\overline{w}_i-v^b_i)^2\\
&\geq\sum_{i\in\mathcal{I}}(\overline{w}_i-v^b_i)^2+N-K\\
&\geq\sum_{i\in\mathcal{I}}(\overline{w}_i-v^b_i)^2+\sum_{i\notin\mathcal{I}}(\overline{w}_i-v^b_i)^2\\
&=\|\overline{\bm{W}}-\bm{V}^b\|^2
\end{aligned}
\end{equation}
Bring \eqref{lemma2_2} to \eqref{lemma2_1}, it yields
\begin{equation}\label{lemma2_3}
(\bm{W}^b-\bm{V}^b)^T(\overline{\bm{W}}-\bm{V}^b)\geq \sqrt{\frac{K}{N}\|\bm{W}^b_i-\bm{V}^b\|^2\|\overline{\bm{W}}-\bm{V}^b\|^2},
\end{equation}
which completes the proof of Lemma \ref{lemma::1}.
\subsection{Proof of Lemma \ref{lemma::2}}\label{proof_lemma::2}
By applying Lemma 1, and we have
\begin{equation}\label{lemma3_1}\begin{aligned}
F(\bm{W}^b_t)-F(\bm{W}^*)\leq&
\nabla F(\bm{W}^b_t)^T(\bm{W}^b_t-\bm{W}^*)\\
&-\frac{1}{2\beta}\|\nabla F(\bm{W}^b_t)-\nabla F(\bm{W}^*)\|^2,
\end{aligned}
\end{equation}
\begin{equation}\label{lemma3_2}\begin{aligned}
F(\bm{W}^*)-F(\bm{W}^b_t)\leq&\nabla F(\bm{W}^*)^T(\bm{W}^*-\bm{W}^b_t)\\
&-\frac{1}{2\beta}\|\nabla F(\bm{W}^*)-\nabla F(\bm{W}^b_t)\|^2.
\end{aligned}
\end{equation}
Adding \eqref{lemma3_1} and \eqref{lemma3_2}, it yields
\begin{equation}\label{lemma3_3}\begin{aligned}
(\nabla F(\bm{W}^b_t)-\nabla F(\bm{W}^*))^T&(\bm{W}^b_t-\bm{W}^*)\geq\\
&\frac{1}{\beta}\|\nabla F(\bm{W}^b_t)-\nabla F(\bm{W}^*)\|^2.
\end{aligned}
\end{equation}
Since $F(\cdot)$ has $\xi$-strong convexity, we can obtain
\begin{equation}\label{lemma3_3_1}
\begin{aligned}
F(\bm{W}^b_t)-F(\bm{W}^*)\leq&\nabla F(\bm{W}^b_t)^T(\bm{W}^b_t-\bm{W}^*)\\
&-\frac{\xi}{2}\|\bm{W}^b_t-\bm{W}^*\|^2,
\end{aligned}
\end{equation}
\begin{equation}
\begin{aligned}
F(\bm{W}^*)-F(\bm{W}^b_t)\leq&\nabla F(\bm{W}^*)^T(\bm{W}^*-\bm{W}^b_t)\\
&-\frac{\xi}{2}\|\bm{W}^*-\bm{W}^b_t\|^2,
\end{aligned}
\end{equation}
and, hence,
\begin{equation}\label{lemma3_4}
(\nabla F(\bm{W}^b_t)-\nabla F(\bm{W}^*))^T(\bm{W}^b_t-\bm{W}^*)\geq\xi\|\bm{W}^b_t-\overline{\bm{W}}_t\|^2
\end{equation}
By combining \eqref{lemma3_3} and \eqref{lemma3_4}, we have
\begin{equation}\label{lemma3_5}\begin{aligned}
(\nabla F(\bm{W}^b_t)&-\nabla F(\bm{W}^*))^T(\bm{W}^b_t-\bm{W}^*)\geq\\
&\sqrt{\frac{\xi}{\beta}}\|\nabla F(\bm{W}^b_t)-\nabla F(\bm{W}^*)\|\|\bm{W}^b_t-\bm{W}^*\|,
\end{aligned}
\end{equation}
since the right side term of \eqref{lemma3_5} is the geometric mean of those of \eqref{lemma3_3} and \eqref{lemma3_4}. Note that $\cos\langle\bm{a},\bm{b}\rangle=\frac{\bm{a}^T\bm{b}}{\|\bm{a}\|\|\bm{b}\|}$, we can rewrite \eqref{lemma3_5} as
\begin{equation}\label{lemma3_6}
|\langle\nabla F(\bm{W}^b_t)-\nabla F(\bm{W}^*),\bm{W}^b_t-\bm{W}^*\rangle|\leq\arccos{\sqrt{\frac{\xi}{\beta}}}.
\end{equation}
By applying Lemma \ref{lemma::1} and \eqref{lemma3_6}, we have
\begin{equation}
\begin{aligned}
\label{lemma3_7}
&|\langle\nabla F(\bm{W}^b_t)-\nabla F(\bm{W}^*),\overline{\bm{W}}_t-\bm{W}^*\rangle|\\
&\leq|\langle\nabla F(\bm{W}^b_t)-\nabla F(\bm{W}^*),\bm{W}^b_t-\bm{W}^*\rangle|\\
&+|\langle\overline{\bm{W}}_t-\bm{W}^*,\bm{W}^b_t-\bm{W}^*\rangle|\\
&\leq\arccos{\sqrt{\frac{\xi}{\beta}}}+\arccos\sqrt{\frac{K}{N}}.
\end{aligned}
\end{equation}
Note that $\|\nabla F(\bm{W}^b_t)\| = \lambda_t\|\nabla F(\bm{W}^*)\|$. Hence,
\begin{equation}\begin{aligned}
\label{lemma3_8}
|\langle \nabla F(\bm{W}^b_t),\nabla F(\bm{W}^b_t)-&\nabla F(\bm{W}_t^*)\rangle| \leq \arctan \frac{1}{\lambda_t}\\
\frac{\lambda_t-1}{\lambda_t}\|\nabla F(\bm{W}^b_t)\| &\leq\|\nabla F(\bm{W}^b_t)-\nabla F(\bm{W}^*)\| \\
&\leq \frac{\lambda_t+1}{\lambda_t}\|\nabla F(\bm{W}^b_t)\|.
\end{aligned}
\end{equation}
By combining \eqref{lemma3_7} and \eqref{lemma3_8}, it yields
\begin{equation}
\begin{aligned}
\label{lemma3_9}
&|\langle\nabla F(\bm{W}^b_t),\overline{\bm{W}}_t-\bm{W}^*\rangle|\\
&\leq|\langle\nabla F(\bm{W}^b_t),\overline{\bm{W}}_t-\bm{W}^*\rangle|\\
&+|\langle\nabla F(\bm{W}^b_t),\nabla F(\bm{W}^b_t)-\nabla F(\bm{W}^*)\rangle|\\
&=\arccos{\sqrt{\frac{\xi}{\beta}}}+\arccos\sqrt{\frac{K_t}{N}}+\arccos\frac{\lambda_t}{\sqrt{\lambda_t^2+1}}\\
&=\arccos\frac{\lambda_t}{\lambda_t-1}\phi(\lambda_t, K_t).
\end{aligned}
\end{equation}
From the definition of $\beta$-smooth, the property of $\epsilon$-strong convex, and \eqref{lemma3_8}, we can easily get
\begin{equation}\label{lemma3_9_1}
\frac{\xi\lambda_t}{\lambda_t+1}\leq
\frac{\|\nabla F(\bm{W}_t^b)\|}{\|\bm{W}_t^b-\bm{W}^*\|}
\leq\frac{\beta\lambda_t}{\lambda_t-1},
\end{equation}
which can be easily transformed to the conclusion 1) in Lemma \ref{lemma::2}.

Now, recall that $\|\bm{W}^b_t-\bm{W}^*\|=2\sqrt{K_t}$, the distance between $\overline{\bm{W}}_{t+1}$ and $\bm{W}^*$ is given as follows
\begin{equation}\begin{aligned}\label{lemma3_10}
&\|\overline{\bm{W}}_{t+1}-\bm{W}^*\|^2\\
=& \|\overline{\bm{W}}_t-\eta\nabla F(\bm{W}^b_t)-\bm{W}^*\|^2\\
=&\|\overline{\bm{W}}_t-\bm{W}^*\|^2+\eta^2\|\nabla F(\bm{W}^b_t)\|^2\\
&-2\eta\nabla F(\bm{W}^b_t)^T(\overline{\bm{W}}_t-\bm{W}^*)\\
=&\|\overline{\bm{W}}_t-\bm{W}^*\|^2+\eta^2\|\nabla F(\bm{W}^b_t)\|^2-2\eta\|\nabla F(\bm{W}^b_t)\|*\\
&\|\overline{\bm{W}}_t-\bm{W}^*\|\cos\langle\nabla F(\bm{W}^b_t),\overline{\bm{W}}_t-\bm{W}^*\rangle\\
\overset{a}{<}&\|\overline{\bm{W}}_t-\bm{W}^*\|^2+\eta^2\|\nabla F(\bm{W}^b_t)\|^2\\
&-\frac{2\phi(\lambda_t,K_t)\eta\beta\lambda_t}{\lambda_t-1}\|\nabla F(\bm{W}^b_t)\|\|\overline{\bm{W}}_t-\bm{W}^*\|\\
\overset{b}{\leq}&\|\overline{\bm{W}}_t-\bm{W}^*\|^2+\eta^2\|\nabla F(\bm{W}^b_t)\|^2\\
&-\frac{\phi(\lambda_t,K_t)\eta\beta\lambda_t}{\lambda_t-1}\|\nabla F(\bm{W}^b_t)\|\|\bm{W}^b_t-\bm{W}^*\|\\
\overset{c}{\leq}&\|\overline{\bm{W}}_t-\bm{W}^*\|^2-\frac{(\phi(\lambda_t,K_t)-\eta)\eta\beta\lambda_t}{\lambda_t-1}*\\
&\|\nabla F(\bm{W}^b_t)\|\|\bm{W}^b_t-\bm{W}^*\|\\
\overset{d}{\leq}&\|\overline{\bm{W}}_t-\bm{W}^*\|^2-\frac{4(\phi(\lambda_t,K_t)-\eta)\eta\beta\xi\lambda_t^2 K_t}{\lambda_t^2-1}\\
<&\|\overline{\bm{W}}_t-\bm{W}^*\|^2-4(\phi(\lambda_t,K_t)-\eta)\eta\beta\xi K_t,
\end{aligned}
\end{equation}
where $a$ is derived from \eqref{lemma3_9} and the definition of $\delta$, $b$ is derived by applying Lemma \ref{lemma::1}, $c$ and $d$ comes from \eqref{lemma3_9_1}. The proof of conclusion 2) in lemma \ref{lemma::2} is completed.

Thus, the proof of Lemma \ref{lemma::2} has been completed.
\subsection{Proof of Theorem \ref{lemma::3}}\label{proof_lemma::3}
Recall that $K_t$ is the amount of different elements in $\bm{W}_t^b$ and $\bm{W}^*$. Obviously, when $K_t\geq K_0+\kappa$,
\begin{equation}
\|\overline{\bm{W}}_t-\bm{W}^*\|^2<K_t*2^2+(N-K_t)*1^2
<(3+N/K_0)K_t.
\end{equation}
From Lemma \ref{lemma::2},
\begin{equation}\begin{aligned}\label{lemma4_1}
&\|\overline{\bm{W}}_{t+1}-\bm{W}^*\|^2\\
<&\|\overline{\bm{W}}_t-\bm{W}^*\|^2-4(\phi(\lambda_t,K_t)-\eta)\eta\beta\xi K_t\\
<&(1-\frac{4(\phi(\lambda_0,K_t)-\eta)\eta\beta\xi}{3+N/K_0})\|\overline{\bm{W}}_t-\bm{W}^*\|^2\\
\triangleq&(1-C_1(\phi(\lambda_0,K_t)-\eta))\|\overline{\bm{W}}_t-\bm{W}^*\|^2\\
\leq&(1-C_1(\phi(\lambda_0, K_0+\kappa)-\phi(\lambda_0,K_0)))\|\overline{\bm{W}}_t-\bm{W}^*\|^2\\
\triangleq&(1-C_2)\|\overline{\bm{W}}_t-\bm{W}^*\|^2,
\end{aligned}
\end{equation}
where $C_1$ and $C_2$ are positive constants{, and $C_2<1$ since the right hand cannot be negative}.

{ We then discuss the upper bound of $T_0$. We recall that $\bm{W}^*$ is the set of the optimal binary parameters. From Lemma 2, $K_t=\|\bm{W}^b_t-\bm{W}^*\|/2\leq\|\overline{\bm{W}}_t-\bm{W}^*\|$, which means that if $\|\overline{\bm{W}}_t-\bm{W}^*\|<K_0+\kappa$, then $K_t<K_0+\kappa$. Hence, for $t_0=2(\log(K_0+\kappa)-\log(\|\overline{\bm{W}}_0-\bm{W}^*\|))/\log(1-C_2)$, if the constraint of Lemma \ref{lemma::2} always holds before $t_0$, we have $\|\overline{\bm{W}}_{t_0}-\bm{W}^*\|\overset{(\ref{lemma4_1})}{<}(1-C_2)^{t_0/2}\|\overline{\bm{W}}_0-\bm{W}^*\|=K_0+\kappa$, and hence $K_{t_0}<K_0+\kappa$. If the constraint of Lemma \ref{lemma::2} does not hold at some time $t_1$ before $t_0$, we have $K_{t_1}<K_0$. Both case indicate that $T_0\leq t_0$. Hence, we have $T_0\leq t_0=\mathcal{O}(\log\frac{1}{K_0+\kappa})$, which ends the proof of Theorem \ref{lemma::3}.
}
{
\subsection{Proof of Theorem 2}\label{proof_lemma::4}
From Theorem 1, we have proved that the gradient descent algorithm for binary models provides geometric convergence, which is the same with real-valued neural networks in the overparameterized setting with constant step size defined in \cite{Qu20}. The only difference is that the model can only converge to a set of suboptimal parameters in the binary case. Since we have the assumption that each parameter $\overline{w}$ is updated with an unbiased estimation, the expectation will not be changed during the parameter aggregating.

Hence, from Theorem 5 in \cite{Qu20}, we have
\begin{equation}\begin{aligned}
&\mathbb{E} F\left(\mathbf{W}_{m,T}\right) - F\left(\bm{W}^*\right) \leq\\ &\mathcal{O}\big(\beta \exp(-\frac{\mu}{E} \frac{M T}{\beta_0 \nu_{\max }+\beta\left(M-\nu_{\min }\right)})\\ &\cdot\left\|\mathbf{W}_{m,0}-\mathbf{W}^{*}\right\|^{2}\big),
\end{aligned}\end{equation}
when we choose a reasonable learning rate and Lemma 3 holds. Here,  $F\left(\mathbf{W}\right)\triangleq \frac{1}{|\mathcal{D}|}\sum_{i=1}^M |\mathcal{D}|F_i\left(\mathbf{W}\right)$ is the total loss function of FL, and $F_i\left(\mathbf{W}\right)=\frac{1}{|\mathcal{D}_i|}\sum_{j=1}^{|\mathcal{D}_i|}\ell(y_{ij},g(\bm{x}_ij,\bm{W}))$. $\ell(\cdot)$ is $\beta_0$-smooth and $F(\cdot)$ is $\beta$-smooth. $\nu_{\max}$ and $\nu_{\min}$ are the maximum and minimum value of $M\frac{|\mathcal{D}_i|}{|\mathcal{D}|}$, respectively. $E$ is the number of epochs conducted among each aggregating, hence $E=1$ here.

From the definition of $\xi$-strong convexity, we can easily know that
\begin{equation}
\begin{aligned}
&\frac{\xi}{2}\|\mathbf{W}_{m,T}-\bm{W}^*\|^2-\epsilon\|\mathbf{W}_{m,T}-\bm{W}^*\|\\
\leq &\frac{\xi}{2}\|\mathbf{W}_{m,T}-\bm{W}^*\|^2+\nabla F(\bm{W}^*)^T(\mathbf{W}_{m,T}-\bm{W}^*)\\
\leq &\|F\left(\mathbf{W}_{m,T}\right) - F\left(\bm{W}^*\right)\|\\
\leq &\mathcal{O}\left( \exp M T\right).
\end{aligned}
\end{equation}

Since $\epsilon$ is a small positive value, we have $\|\mathbf{W}_{m,T}-\bm{W}^*\|\leq\mathcal{O}\left( \exp M T\right)$ when Lemma 3 holds. Hence, with $t_0=\mathcal{O}(\frac{1}{M}\log\frac{1}{K_0+\kappa})$, either $\|\bm{W}_{m,t}^b-\bm{W}^*\|/2<K_0+\kappa$ or Lemma 3 does not hold at a time before $t_0$. Both cases indicate that Theorem 2 holds.
}
\end{appendices}

\normalem
\printbibliography
\end{document}